\documentclass[10pt]{article} %

\usepackage[preprint]{arxiv}

\usepackage{microtype}
\usepackage{graphicx}
\usepackage{booktabs} %

\usepackage[usenames,dvipsnames]{xcolor}
\usepackage[colorlinks=true,citecolor=ForestGreen,linkcolor=Bittersweet]{hyperref}

\usepackage{amsmath}
\usepackage{amssymb}
\usepackage{mathtools}
\usepackage{amsthm}
\usepackage{dsfont}
\usepackage{pifont}
\usepackage[capitalize,noabbrev]{cleveref}

\usepackage{multirow}
\usepackage{colortbl}
\usepackage{makecell}
\usepackage{float}

\usepackage{newtxtext}

\newif\ifarxiv
\arxivtrue

\newcommand{\arxivGuard}[1]{}

\usepackage[toc,page,header]{appendix}
\usepackage{minitoc}
\newif\ifaddtoc

\addtocfalse

\ifaddtoc
    \usepackage[toc,page,header]{appendix}
    \usepackage{minitoc}

\fi

\usepackage{enumitem}
\usepackage[mathscr]{eucal}
\usepackage{amsmath,amsfonts,bm}

\usepackage[utf8]{inputenc}
\usepackage{array}
\usepackage{amssymb}
\usepackage{amsbsy}
\usepackage{microtype}

\usepackage{subcaption}
\usepackage{graphicx}
\usepackage[export]{adjustbox}
\usepackage{tcolorbox}

\usepackage{wrapfig}

\usepackage[normalem]{ulem}

\usepackage[textsize=tiny,textwidth=1.9cm]{todonotes}
\usepackage{pifont}
\usepackage{soul}
\usepackage{xfrac}

\def\1{\bm{1}}
\newcommand{\best}[1]{\cellcolor{gray!25}{#1}}

\newcommand{\pLM}{p}

\newcommand{\STrain}{S_{\text{tr}}}
\newcommand{\SVal}{S_{\val}}

\newcommand{\NTrain}{N_{\text{tr}}}
\newcommand{\NVal}{N_{\val}}

\newcommand{\QueryEmbed}{\boldsymbol{\Phi}}

\newcommand{\PretrainedQueryEmbed}{\boldsymbol{\varphi}}
\newcommand{\PretrainedQueryEmbedDim}{D_{\text{P}}}

\newcommand{\LLMEmbed}{\boldsymbol{\Psi}}
\newcommand{\LLMEmbedScalar}{\Psi}

\newcommand{\LLMClusterEmbedH}{{\LLMEmbed}_{\text{clust}}( h )}

\newcommand{\LLMClusterEmbedScalarHK}{{\LLMEmbedScalar}_{\text{clust}, k}( h )}

\newcommand{\InputClusterEmbed}{\QueryEmbed_{\text{clust}}( \boldsymbol{x} )}
\newcommand{\InputLearnedClusterEmbed}{\QueryEmbed_{\text{clust}}( \boldsymbol{x}; \boldsymbol{\theta} )}

\newcommand{\InputEmbedDim}{K}
\newcommand{\LLMEmbedDim}{K}

\newcommand{\ClusterMembership}{\QueryEmbed_{\text{clust}}}
\newcommand{\ClusterMembershipScalar}[1]{\QueryEmbed_{\text{clust}, {#1}}}

\newcommand{\LossVector}{prediction error vector}
\newcommand{\LossVectorSentenceCaps}{Prediction Error Vector}

\newcommand{\gammaEstimator}{\gamma}

\newcommand{\routerEstimator}{r}

\theoremstyle{plain}
\newtheorem{thm}{\protect\theoremname}
\theoremstyle{plain}
\newtheorem{prop}[thm]{\protect\propositionname}
\newtheorem*{prop*}{\protect\propositionname}
\theoremstyle{plain}
\newtheorem{lem}[thm]{\protect\lemmaname}

\theoremstyle{plain}
\newtheorem{theorem}{Theorem}[section]
\newtheorem{lemma}[theorem]{Lemma}\theoremstyle{definition}
\theoremstyle{remark}

\newcommand{\boldcheck}{\pmb{\checkmark}}
\newcommand{\crossmark}{\ding{56}}

\newcommand{\todowj}[1]{}
\newcommand{\todoasr}[1]{}
\newcommand{\todohari}[1]{}
\newcommand{\wjtext}[1]{#1}
\newcommand{\harichange}[1]{#1}
\newcommand{\asrnote}[1]{}

\newcommand{\hla}[1]{%
  \begingroup\setlength{\fboxsep}{1pt}%
  \colorbox{green!20!white}{{\hspace*{1pt}#1\hspace*{1pt}}}%
  \endgroup
}

\newcommand{\mset}{\mathscr{H}}
\newcommand{\msetAll}{\mset_{\rm all}}
\newcommand{\msetTrain}{\mset_{\text{tr}}}

\newcommand{\hOld}{h_{\mathrm{tr}}}
\newcommand{\hNew}{h_{\mathrm{te}}}
\newcommand{\msetNew}{\mathscr{H}_{\mathrm{te}}}
\newcommand{\numSetNew}{N}

\newcommand{\cluster}{\text{clust}}

\newcommand{\ourMethod}{{\tt UniRoute}}

\makeatother

\providecommand{\propositionname}{Proposition}
\providecommand{\lemmaname}{Lemma}
\providecommand{\theoremname}{Theorem}

\graphicspath{{figs//}}

\allowdisplaybreaks

\newcommand{\removefornow}[1]{}

\renewcommand{\Pr}{\mathbb{P}}

\newcommand{\XCal}{\mathscr{X}}

\newcommand{\defEq}{\stackrel{.}{=}}
\newcommand{\defeq}{\defEq}
\newcommand{\Real}{\mathbb{R}}

\newcommand{\bx}{\boldsymbol{x}}
\newcommand{\by}{\boldsymbol{y}}

\newcommand{\argmin}{\operatorname{argmin}}

\theoremstyle{definition}

\newcommand{\tr}{\textrm{tr}}

\newcommand{\val}{\textrm{val}}

\author{%
  \large{Wittawat Jitkrittum} \\
  \texttt{wittawat@google.com} \\
  \And
  \large{Harikrishna Narasimhan} \\
  \texttt{hnarasimhan@google.com} \\
  \And
  \large{Ankit Singh Rawat} \\
  \texttt{ankitsrawat@google.com} \\
  \And  
  \large{Jeevesh Juneja} \\
  \texttt{jeeveshjuneja@google.com} \\
  \And    
  \large{Congchao Wang} \\
  \texttt{congchaowang@google.com} \\
  \And
  \large{Zifeng Wang} \\
  \texttt{zifengw@google.com} \\
  \And      
  \large{Alec Go} \\
  \texttt{ago@google.com} \\
  \And      
  \large{Chen-Yu Lee} \\
  \texttt{chenyulee@google.com} \\
  \And        
  \large{Pradeep Shenoy} \\
  \texttt{shenoypradeep@google.com} \\
  \And        
  \large{Rina Panigrahy} \\
  \texttt{rinap@google.com} \\
  \And          
  \large{Aditya Krishna Menon} \\
  \texttt{adityakmenon@google.com} \\
  \And
  \large{Sanjiv Kumar} \\
  \texttt{sanjivk@google.com} \\
}
\affiliation{Google}

\title{\LARGE\sffamily Universal Model Routing for Efficient LLM Inference}

\begin{document}

\doparttoc %
\faketableofcontents %

\maketitle

\begin{abstract}
\emph{Model routing} is a simple technique for reducing 
the inference cost
of large language models (LLMs),
wherein one maintains a pool of candidate LLMs,
and learns to route each prompt to the smallest feasible LLM.
Existing works focus on learning a router for a \emph{fixed} pool of
LLMs.
In this paper, we consider the problem of \emph{dynamic} routing,
where \emph{new, previously unobserved} LLMs are available at test time.
We propose 
\ourMethod{}, 
a new approach to this problem that relies on representing each LLM as a \emph{feature vector}, derived based on predictions on a set of representative prompts.
Based on this, we detail two effective instantiations of \ourMethod{}, relying on \emph{cluster-based} routing and a \emph{learned cluster map} respectively.
We show that these are estimates of a theoretically optimal routing rule, and quantify their errors via an excess risk bound.
Experiments on a range of public benchmarks show the effectiveness of \ourMethod{} in routing amongst more than 30 unseen LLMs.
\end{abstract}

\section{Introduction}
Large language models (LLMs) have seen a flurry of 
recent development~\citep{Radford:2018,Radford:2019,Brown:2020,Touvron:2023,Anil:2023,Grattafiori:2024,DeepSeekAI:2024}.
These impressive abilities notwithstanding,
the inference cost of LLMs can be prohibitive~\citep{Li:2024e,Wan:2024,Zhou:2024b}.
This has motivated 
several
techniques to 
improve 
LLM inference efficiency,
such as
speculative decoding~\citep{Stern:2018,chen2023accelerating,leviathan2023fast}, 
early-exiting~\citep{SchFisGup2022}, 
quantisation~\citep{Chee:2023}, 
compression~\citep{Frantar:2023,Agarwal:2024,Rawat:2024},
and others~\citep{Pope:2023,Aishwarya:2024,Menghani:2022}.

Our interest is in \emph{model routing} for efficient inference.
Here,
one maintains a pool of candidate LLMs of various sizes and capabilities.
Given a prompt, 
one learns to predict the lowest-cost LLM which can reasonably address the prompt.
In doing so, one can learn to use high-cost LLMs sparingly, only on the (relatively) few ``hard'' inputs.
This is a conceptually simple but effective technique,
which has seen a surge of recent interest~\citep{Hendy:2023,Narayanan:2023,Ding:2024,Sakota:2024,CheJiaLin2024,HuBieLi2024,Shnitzer:2023,Wang:2023,Stripelis:2024,OngAlmWu2024,ZhuWuWen2024,Srivatsa:2024,Feng:2024,LuYuaLin2024,Zhao:2024,Dann:2024,Aggarwal:2024,Lee:2024,Mohammadshahi:2024,Chuang:2025,Huang:2025}.

Existing works largely focus on routing over a \emph{fixed} pool of LLMs.
In practice, however, the pool of candidate LLMs can constantly change;
e.g., older LLMs may be deprecated in favor of new, performant LLMs.
To leverage such new LLMs,
perhaps the simplest approach is to re-train the router.
However, 
with frequent changes to the LLM pool,
this may be impractical 
owing to the non-trivial overhead
of 
both model re-training,
as well as 
obtaining sufficient \emph{training labels} for each new LLM.

In this paper, 
we propose 
\ourMethod{},
a new approach
to
this problem 
based on 
representing each LLM as a \emph{feature vector}, derived from 
their
\emph{prediction errors}
on a set of representative prompts.
By learning a router over these LLM features, 
we enable generalisation to previously unseen LLMs \emph{without} any re-training.
Building on the observation
that $K$-NN routing~\citep{HuBieLi2024} is a special case of \ourMethod{},
we detail two 
concrete 
instantiations of \ourMethod{}
relying on 
\emph{unsupervised} and \emph{supervised} prompt clustering respectively.
While conceptually simple,
empirically,
these 
enable 
effective 
routing
with a dynamic LLM pool across several benchmarks.
In sum, our contributions are:

\begin{enumerate}[label=(\roman*),itemsep=0pt,topsep=0pt,leftmargin=16pt]
    \item we 
    formalise the
    problem setting of model routing with a \emph{dynamic} pool of LLMs (\S\ref{sec:dynamic_routing});

    \item we propose
    \ourMethod{} (\S\ref{sec:uniroute-framework}),
    a novel
    routing
    approach
    relying on 
    an LLM feature representation based on its 
    \emph{\LossVector{}} 
    on a small set of representative prompts (\S\ref{sec:correctness_representation});
    
    \item 
    we propose two %
    simple yet effective
    instantiations of \ourMethod{} 
    based on unsupervised or supervised clustering  (\S\ref{sec:cluster_router}, \S\ref{sec:two_tower}, \harichange{Figure \ref{fig:illus_cluster}}), with an accompanying excess risk bound (\S\ref{sec:excess-risk});

    \item we present experiments 
    (\S\ref{sec:experiments})
    on 
    EmbedLLM \citep{ZhuWuWen2024},
    SPROUT o3-mini \citep{Somerstep:2025},
    RouterBench~\citep{HuBieLi2024}, and Chatbot Arena~\citep{OngAlmWu2024}, 
    illustrating the ability to
    effectively route amongst $>30$ unseen LLMs.
\end{enumerate}

\begin{figure}[!t]
    \centering
    \includegraphics[width=0.9\textwidth]{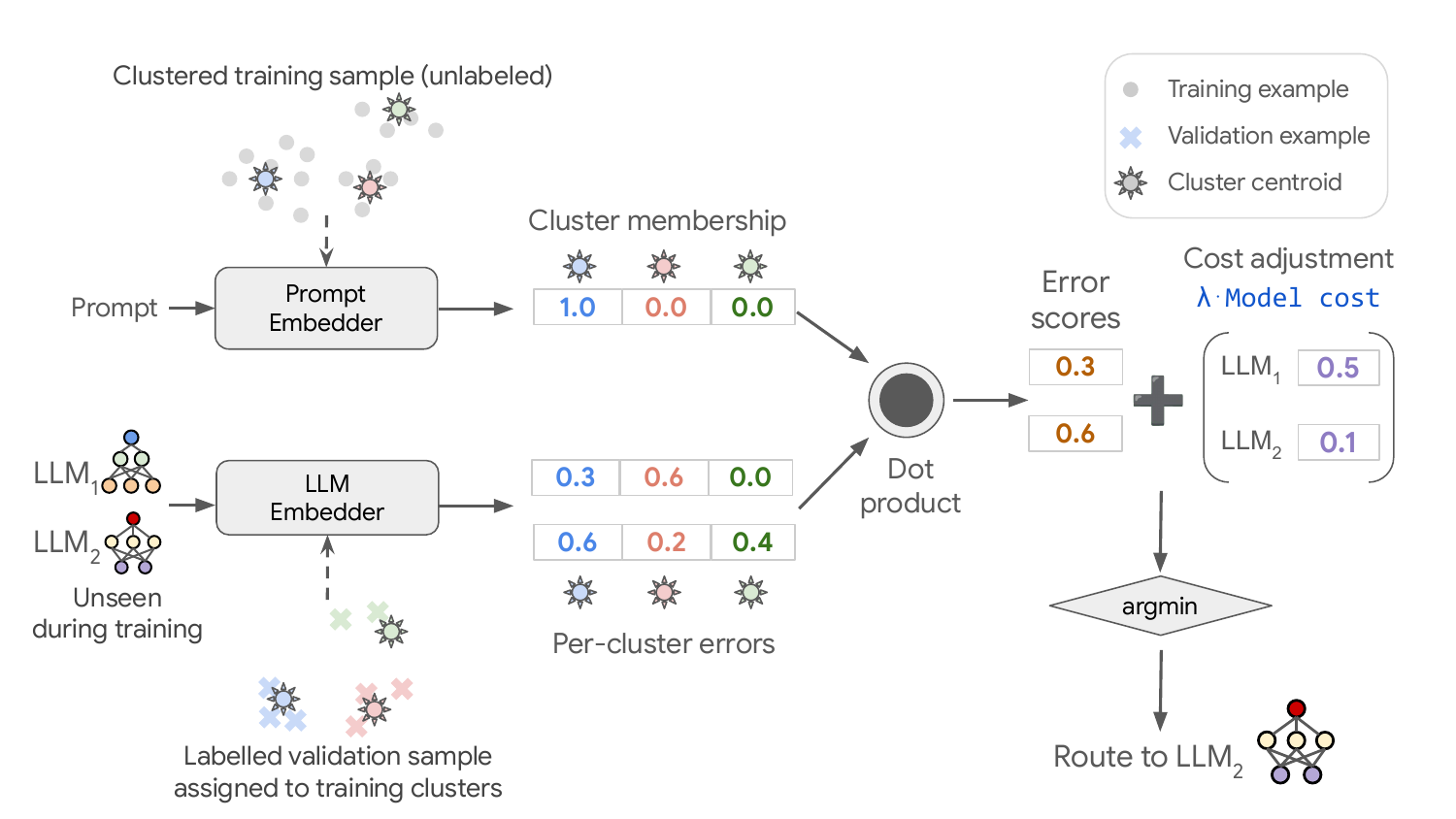}
    \caption{Illustration of our proposed 
    \ourMethod{} with a 
    cluster-based router (see \S\ref{sec:cluster_router}). 
    We first perform $K$-means on 
    a
    training set to find $K$ centroids, 
    and then
    partition the validation set into $K$ representative clusters. 
    Each test-time LLM can then be represented as a $K$-dimensional feature vector of per-cluster errors. 
    This yields an intuitive routing rule:
    for each test prompt, 
    we route to the LLM with the smallest cost-adjusted average error on the cluster the prompt belongs to. The prompt embedder may either be completely \emph{unsupervised} (as shown in the figure), or fitted via \emph{supervised} learning using labels from a set of training  LLMs different from those seen during test time (\S\ref{sec:two_tower}).
    }
    \vspace{-10pt}
    \label{fig:illus_cluster}
\end{figure}

\section{Background: Model Routing with a Static LLM Pool}
\label{sec:background}

\textbf{Language models (LMs)}.
Given a finite, non-empty vocabulary of \emph{tokens} $\mathscr{V}$,
a \emph{language model} (\emph{LM})
is a distribution 
$\pLM \in \Delta( \mathscr{V}^* )$,
where
$\mathscr{V}^* \defEq \bigcup_{n = 0}^{\infty} \mathscr{V}^n$
and $\Delta( \cdot )$ denotes the set of distributions over a set.
\emph{Large} language models (\emph{LLMs})
based on Transformers~\citep{Vaswani:2017}
have proven highly versatile.

\textbf{LLMs for predictive tasks}.
Our interest is in using LLMs for the following prediction problem.
Let $\mathscr{X} \subset \mathscr{V}^*$ be a set of \emph{input prompts},
and $\mathscr{Y}$ be a set of \emph{targets}.
Let
$\ell( \boldsymbol{x}, \boldsymbol{y}, \hat{\boldsymbol{y}} )$
denote a \emph{loss function}
measuring
the \emph{loss} (or \emph{disutility}) of a \emph{predicted} response $\hat{\boldsymbol{y}}$ on a given (prompt, target) pair $( \boldsymbol{x}, \boldsymbol{y} )$.
For example, 
$\ell( \boldsymbol{x}, \boldsymbol{y}, \hat{\boldsymbol{y}} ) = \1( \boldsymbol{y} \neq \hat{\boldsymbol{y}} )$ 
measures whether the response is an exact string match of the target.
Our goal is to construct a \emph{predictor}
$h \colon \mathscr{X} \to \mathscr{Y}$ 
minimising
$R( h ) \defEq \mathbb{E}_{( \boldsymbol{x}, \boldsymbol{y} ) \sim \Pr}\left[ \ell( \boldsymbol{x}, \boldsymbol{y}, h ) \right]$.

An LLM natively provides a distribution over $\mathscr{V}^*$.
To
convert
such a distribution to a predicted target,
we assume there is some (task-specific, and possibly randomised) \emph{prediction function} $\texttt{predict} \colon \Delta( \mathscr{V}^* ) \to \mathscr{Y}$;
e.g., in the simple case where $\mathscr{Y} \subset \mathscr{V}^*$, 
we may employ a standard decoding 
algorithm~\citep{Ficler:2017,Fan:2018,Holtzman:2020}.
More generally, $\texttt{predict}$ may involve some non-trivial processing 
(e.g., stripping non-numeric tokens).
Given such a function, we may construct
$h( \boldsymbol{x} ) \defEq \texttt{predict}( \pLM( \cdot \mid \boldsymbol{x} ) )$
to minimise $R( h )$.\asrnote{Should we highlight here what predict() looks like for us throughout this work? Or is it changing depending on the tasks/benchmarks later on?}

\textbf{Model routing}.
Model routing is a means for achieving efficiency at inference time
by selecting an appropriate LLM 
for each input prompt. 
Inference efficiency is gained by sparingly calling a large
model only on ``hard'' input prompts.
More precisely,
suppose we have a set 
of $M \geq 2$ LLMs
$\{ \pLM^{(m)} \}_{m \in [M]}$,
with corresponding \emph{inference costs}
$\{ c^{(m)} \}_{m \in [M]}$
denoting,
e.g.,
their average monetary costs for processing one prompt.
We assume $c^{(1)} \leq c^{(2)} \leq \ldots \leq c^{(M)}$.
Let $r \colon \mathscr{X} \to [ M ]$ be a \emph{router} that,
given a prompt, predicts the most suitable LLM.
In the standard (``static'') routing problem, we seek a router which achieves (cf.~\citet{CheZahZou2023,Dekoninck:2025,Woisetschlager:2025,Somerstep:2025})
\begin{equation}
    \resizebox{0.935\linewidth}{!}{$\displaystyle
    \min_{r \colon \mathscr{X} \to [ M ]} 
    \sum_{m \in [ M ]} \mathbb{E}_{( \boldsymbol{x}, \boldsymbol{y} )}\left[ \1( r( \boldsymbol{x} ) = m ) \cdot \ell( \boldsymbol{x}, \boldsymbol{y}, h^{(m)} ) \right] \colon 
    \label{eqn:static-router}
    \sum_{m \in [ M ]} \mathbb{E}_{( \boldsymbol{x}, \boldsymbol{y} )}\left[ \1( r( \boldsymbol{x} ) = m ) \cdot c^{(m)} \right] \leq B.%
    $}
\end{equation}
Here, $B \geq 0$ is some fixed budget on the cost of the routed solution.
We use $h^{(m)}( \boldsymbol{x} ) \defEq \texttt{predict}( \pLM^{(m)}( \cdot \mid \boldsymbol{x} ) )$ for some fixed prediction function $\texttt{predict}$.

\textbf{Model routing strategies}.
The most na\"{i}ve routing strategy \emph{randomly} assigns prompts to the various models,
potentially pruning inadmissible solutions (see Appendix~\ref{app:pareto-random}).
A more refined strategy is to \emph{learn} a router.
\citet{HuBieLi2024} proposed an intuitive strategy 
(see also~\citet{NarJitMen2022}),
wherein
one constructs a predictor 
$\gammaEstimator^{(m)} \colon \mathscr{X} \to \Real_+$ of the expected loss incurred by each LLM $h^{(m)}$,
and routes via
\begin{equation}
    \label{eqn:post-hoc}
    r( \boldsymbol{x} ) = \underset{m \in [M]}{\argmin} \, \left[ \gammaEstimator^{(m)}( \boldsymbol{x} ) + \lambda \cdot c^{(m)} \right].
\end{equation}
Here, 
$\lambda \geq 0$ is a hyper-parameter trading between cost and quality.
In \S\ref{sec:proposal}, we show formally that the routing rule in \eqref{eqn:post-hoc} is a plug-in estimator of a theoretically optimal routing rule (cf.~also~\citet{Somerstep:2025}).

Given a training sample 
$\STrain \defEq \{ ( \boldsymbol{x}^{(i)}, \boldsymbol{y}^{(i)} ) \}_{i = 1}^{\NTrain}$,
there are several approaches to construct $\gammaEstimator$.
For example, 
given a text embedder 
$\PretrainedQueryEmbed \colon \mathscr{X} \to \Real^{\PretrainedQueryEmbedDim}$
(e.g., BERT~\citep{Devlin:2019}, Sentence-T5~\citep{sentencet5}),
one may employ:
\begin{equation}
    \label{eqn:bert-router}
    \gammaEstimator_{\text{lin}}^{(m)}( \boldsymbol{x} ) = \mathbf{w}_{m}^\top \PretrainedQueryEmbed( \boldsymbol{x} ) + b_m,
\end{equation}
where
$\mathbf{w}_m \in \Real^{\PretrainedQueryEmbedDim}, b_m \in \Real$
and (optionally) $\PretrainedQueryEmbed$
may be fit
with a suitable empirical loss, 
e.g.,
\begin{equation}
    \label{eqn:bert-router-erm}
    \frac{1}{N} \sum_{i \in [ \NTrain ]} \sum_{m \in [ M ]} ( \ell( \bx^{(i)}, \by^{(i)}, h^{(m)} ) - \gammaEstimator_{\text{lin}}^{(m)}( \bx^{(i)} ) )^2.
\end{equation}
A 
related
matrix factorisation approach
operating over \emph{frozen} embeddings
was proposed in~\citet{OngAlmWu2024},
fitted on a sample comprising of either \emph{pairwise} comparisons~\citep{OngAlmWu2024} or pointwise correctness labels \citep{ZhaJinMao2024}.
Alternatively,
a $K$-NN estimator can be used~\citep[Section 5.1]{HuBieLi2024}:
\begin{equation}
    \label{eqn:knn-router}
    \gammaEstimator_{\text{kNN}}^{(m)}( \boldsymbol{x} ) = \frac{1}{k} \sum_{i \in {\rm NN}( \boldsymbol{x}, k )} \ell( \boldsymbol{x}^{(i)}, \boldsymbol{y}^{(i)}, h^{(m)} ),    
\end{equation}
where ${\rm NN}( \boldsymbol{x}, k )$ denotes the $k$-nearest neighbours 
of 
(the embeddings of)
$\boldsymbol{x}$
in $\STrain$.

\section{Model Routing with a Dynamic LLM Pool}
\label{sec:dynamic_routing}

We now formalise the model routing problem when the set of LLMs may vary \emph{dynamically}.

\subsection{Problem Setup}
\label{sec:setup}

The routing setup in~\eqref{eqn:static-router}
assumed a \emph{static} pool of LLMs:
indeed, observe that the linear model~\eqref{eqn:bert-router}
only estimates parameters 
$\{ (\mathbf{w}_m, b_m) \}_{m \in [ M ]}$ for a fixed set of LLMs, namely, $\{ p^{(m)} \}_{m \in [ M ]}$.
In practice, the LLM pool may frequently change,
as new models are released and old models are deprecated.
Na\"{i}vely, one may simply re-train a router 
such as~\eqref{eqn:bert-router} with each such new LLM.
However,
this
requires 
both annotating each training sample with the new LLM's predictions,
several steps of iterative training,
and initiating a fresh router deployment.
For
a constantly refreshing LLM pool,
this can 
impose a 
non-trivial overhead (e.g., computational cost).
This motivates an alternate routing setup.

Concretely, 
let 
$\msetAll$ denote the set of all possible LLM predictors,
where for simplicity we assume $| \msetAll | < +\infty$.
Let $\mathbb{H} \defEq 2^{\msetAll}$ 
denote the set of all subsets of $\msetAll$.
Let
$\mset_{\rm tr} = \{ \hOld^{(1)}, \ldots, \hOld^{(M)} \} \in \mathbb{H}$ 
denote the set of 
LLM predictors observed during training.
During evaluation,
we seek to route amongst the LLM predictors in
$\mset_{\rm te} = \{ \hNew^{(1)}, \ldots, \hNew^{(N)} \} \in \mathbb{H}$.
If $\mset_{\rm tr} = \mset_{\rm te}$, 
we obtain the original routing problem in~\eqref{eqn:static-router}.
However, we aim to allow
$\mset_{\rm tr} \neq \mset_{\rm te}$,
including the case 
$\mset_{\rm tr} \cap \mset_{\rm te} = \varnothing$.

To accommodate such a dynamic LLM pool, 
we first modify our router to accept both an input prompt \emph{and} a set of candidate LLMs,
with the goal to pick the best option from this set;
i.e.,
we consider 
\emph{dynamic routers}
$\mathscr{R} \defEq \{ r( \cdot, {\mset} ) \colon \XCal \to [| \mset |] \mid \mset \in \mathbb{H} \}$.
Next, we assume that the set of 
training LLMs $\msetTrain$ is itself drawn from some \emph{meta-distribution} $\mathfrak{H}$ over $\mathbb{H}$. 
Rather than perform well on the \emph{specific} set $\msetTrain$,
we would like to generalise to \emph{any} set of LLMs drawn from $\mathfrak{H}$.
Concretely, we wish to solve:
\begin{equation}
    \min_{r \in \mathscr{R}}  \mathbb{E} \left[\sum_{m\in[|\mset|]} \1( r( \bx, \mset ) = m ) \cdot \ell(\bx,\by,h^{(m)}) \right] 
    \colon
    \mathbb{E} \left[\sum_{m\in[|\mset|]} \1( r( \bx, \mset ) = m ) \cdot c( h^{(m)} ) \right] \le B,
    \label{eq:risk}
\end{equation}
where 
as before $B \geq 0$ denotes a cost budget,
$\mset \defeq \{ h^{(1)}, \ldots, h^{(M)} \} \sim \mathfrak{H}$ denotes a sample of $M$ LLMs,  
$c \colon \msetAll \to \Real_+$ denotes the cost of a given LLM, and $\mathbb{E}$ is a shorthand for $ \mathbb{E}_{(\bx, \by, \mset)}$.

\subsection{Optimal Routing with a Dynamic Pool}

To guide the design of a dynamic router, 
we begin by studying the \emph{Bayes-optimal} rule for~\eqref{eq:risk}.
Interestingly, we find the Bayes-optimal rule \emph{decomposes} across each of the constituent LLMs.
The result is known
for a \emph{fixed} candidate set $\mset \in \mathbb{H}$~\citep[Lemma F.1]{Jitkrittum:2023},~\citep{Dekoninck:2025,Somerstep:2025}.
The distinction arises for a \emph{varying} candidate set, where the result closely mirrors~\citet[Eq.\ 6]{Tailor:2024}, derived in the related setting of learning to defer to an expert (see~\S\ref{sec:related}).

\begin{prop}[Optimal dynamic routing]
\label{prop:optimal_rule} 
Under a mild regularity condition on $\mathbb{P}$,\todohari{What does our continuity assumption on the distribution mean for discrete distributions over text? To be discussed in appendix}
for any 
input
$\bx \in \XCal$,
LLM candidate set $\mset \in \mathbb{H}$,
and budget $B > 0$,
there exists a
Lagrange multiplier $\lambda_{\mathfrak{H}} \ge 0$ 
such that
the optimal dynamic router $r^{*}$ for the constrained optimization
in \eqref{eq:risk} is
\begin{equation}
\label{eq:opt_rule}
r^{*}(\bx, \mset) = \underset{m\in[ | \mset | ]}{\argmin} \, \left[ \mathbb{E}_{\by\mid\bx}\left[\ell(\bx,\by,h^{(m)})\right]+\lambda_{\mathfrak{H}} \cdot c( h^{(m)} ) \right].
\end{equation}
\end{prop}

Intuitively, 
it is optimal to route to the model that has the lowest expected loss
on the given input $\bx$, 
after applying a \emph{cost adjustment} 
of $\lambda_{\mathfrak{H}} \cdot c( h^{(m)} )$ to the loss.
The hyperparameter $\lambda_{\mathfrak{H}} \ge0$
allows one to trade off the expected quality and the average cost.
If one wishes to \emph{sweep} $B$ to trace a 
cost-quality \emph{deferral curve}
(Appendix~\ref{sec:deferral-curve}),
one may equally treat $\lambda_{\mathfrak{H}}$
as a constant to be swept from $[ 0, +\infty )$.

\paragraph{Special case: 0-1 Loss}
Consider a setting where an LLM response 
may be compared 
to a ground-truth
target
based on an exact match criteria:
i.e., we pick the 0-1 loss
$\ell(\bx,\by,h^{(m)})=\1[h^{(m)}(\bx)\neq\by]$,
with
values either 
0 (incorrect) or 1 (correct).
Here, the optimal rule in \eqref{eq:opt_rule} becomes: %
\begin{align}
r^{*}(\bx, \mset) &= 
\underset{m\in[ |\mset| ]}{\argmin} \, \left[ \gamma^*( \bx, h^{(m)} ) +\lambda_{\mathfrak{H}} \cdot c( h^{(m)} ) \right] \label{eq:opt_rule01} \\
\gamma^*( \bx, h ) &\defEq \mathbb{P}\left[\by\ne h(\bx)\mid\bx\right].\nonumber
\end{align}
For simplicity, we will focus on the the 0-1 loss henceforth, and consider
\eqref{eq:opt_rule01} as the optimal routing rule;
our results can be readily adapted (as we shall see in our experiments in \S\ref{sec:experiments}) to other loss functions
by considering \eqref{eq:opt_rule}.
Example problems where the 0-1 loss are appropriate include
GSM8K \citep{CobKosBav2021}, 
MMLU~\citet{HenBurBas2021}, and problems in 
SuperGLUE~\citep{WanPruNan2019}.

\subsection{Plug-in Routing with a Dynamic Pool}
\label{sec:new_llms}

Proposition~\ref{prop:optimal_rule} and~\eqref{eq:opt_rule01} suggest a simple practical approach to routing with a dynamic 
pool of test LLMs
$\msetNew = \{ \hNew^{(n)} \}_{n \in [N]}$: 
we may construct a plug-in estimator 
$\gammaEstimator( \bx, h )$ of $\gamma^*( \bx, h )$, 
and
estimate~\eqref{eq:opt_rule01} via
\begin{equation}
    \label{eqn:plugin-dynamic-routing}
    \routerEstimator(\bx, \msetNew) = 
    \underset{n \in [ N ]}{\argmin} \, \left[ \gammaEstimator( \bx, \hNew^{(n)} ) +\lambda\cdot c( \hNew^{(n)} ) \right].
\end{equation}
A key question arises: how should we parameterise $\gamma( \bx, h )$?
We study this question next.

\section{\ourMethod{}: Universal Routing via an LLM Feature Representation}
\label{sec:uniroute}
We present a general
dynamic routing approach
based on constructing \emph{LLM feature representations}.

\subsection{The \ourMethod{} Approach}
\label{sec:uniroute-framework}

To enable routing with a dynamic LLM pool,
we propose to parameterise $\gamma$
by
representing both prompts \emph{and} LLMs as feature vectors.
Specifically,
for fixed $K > 1$
let 
$\QueryEmbed \colon \XCal \to \Real^{\InputEmbedDim}$ and
$\LLMEmbed \colon \msetAll \to \Real^{\LLMEmbedDim}$
denote
\emph{feature maps}
for 
prompts and LLMs respectively.
Then, we propose:
\begin{tcolorbox}[
colback=gray!5!white,  %
colframe=gray!5!white, %
arc=3mm,                %
boxsep=0pt,          %
left=2pt,            %
right=2pt,           %
top=1pt,             %
bottom=1pt,          %
toptitle=0pt,        %
bottomtitle=1pt,     %
]
\begin{equation}
    \label{eqn:gamma-uniroute}
    \gammaEstimator_{\text{uni}}( \bx, h ) = \QueryEmbed( \bx )^\top \LLMEmbed( h ).
\end{equation}
\end{tcolorbox}
We may now fit any parameters associated with $\QueryEmbed, \LLMEmbed$ on the training set $\STrain$,
and then
route as before via~\eqref{eqn:plugin-dynamic-routing}.
Crucially,
provided we define an easily computable $\LLMEmbed$,
this
{seamlessly handles}
any
$h \in \msetAll$, \emph{including one unobserved during training};
this is analogous to semantic output codes for zero-shot classification~\citep{Palatucci:2009}.
Thus,~\eqref{eqn:gamma-uniroute} provides an 
approach for \emph{universal routing} with dynamic LLM pools.

To \emph{realise} the potential of this approach,
it now remains to specify 
$\QueryEmbed( \bx )$ and
$\LLMEmbed( h )$.

\textbf{Prompt representation.}\ 
The choice of prompt representations 
$\QueryEmbed( \bx ) \in \Real^{\InputEmbedDim}$
has been well-studied in prior work~\citep{HuBieLi2024}.
A natural choice 
is 
to build on
a general-purpose text embedding
such as 
\texttt{text-embedding-3}~\citep{OpenAI:2025},
NV-Embed~\citep{Lee:2025},
E5-Mistral-7B~\citep{Wang:2024},
or
Gecko~\citep{LeeDaiRen2024}.
Such embeddings may be projected from a native 
$D_{\rm P}$ to $K \ll D_{\rm P}$ dimensional space, e.g., via a linear transformation.

\textbf{LLM representation.}\ 
A good choice for 
$\LLMEmbed( h )$ 
is less apparent than that for $\QueryEmbed( \bx )$.
Observe that
the 
standard
linear router~\eqref{eqn:bert-router}
for a static pool 
$\msetTrain = \{ \hOld^{(1)}, \ldots, \hOld^{(M)} \}$
corresponds to a \emph{one-hot} LLM representation
$\LLMEmbed_{\text{oh}}( h ) = \begin{bmatrix} \1( h = \hOld^{(m)} ) \end{bmatrix}_{m \in [ M ]}$,
and a prompt representation 
$\QueryEmbed( \bx ) = \mathbf{W} \PretrainedQueryEmbed( \bx ) \in \Real^{M}$
for $\mathbf{W} \in \Real^{M \times \PretrainedQueryEmbedDim}$.
As noted previously, such an approach is inherently tied to the LLM pool $\msetTrain$,
analogous to the cold-start problem in collaborative filtering~\citep{Schein:2002}.
A na\"{i}ve 
alternate 
idea 
is to flatten the LLM's trained parameters.
However,
these would be in excess of billions of dimensions 
(exacerbating the risk of overfitting),
and 
are inadmissible for many proprietary LLMs.
We now examine an alternative LLM representation based on \emph{its performance on a subset of prompts}.

\begin{tcolorbox}[
title={\textbf{\ourMethod{}: Routing with a Dynamic LLM Pool}},
halign title=flush center, 
float,floatplacement=!t,
colback=gray!5!white,
colframe=gray!5!white,
colbacktitle=gray!50!white,
arc=3mm,                %
boxsep=0pt,          %
left=2pt,            %
right=2pt,           %
top=2pt,             %
bottom=2pt,          %
toptitle=1pt,        %
bottomtitle=1pt,     %
]
    \begin{enumerate}[label=(\arabic*),leftmargin=16pt]
        \item \textbf{Train base router.}
        Fit parameters of $\gamma_{\text{uni}}$ from~\eqref{eqn:gamma-uniroute} on training set 
        $\STrain = \{ ( \bx^{(i)}, \by^{(i)} ) \}_{i = 1}^{\NTrain}$.
        \item \textbf{Embed test LLMs.}
        Compute $\LLMEmbed( \hNew )$ for each test LLM $\hNew \in \msetNew$, e.g., via~\eqref{eqn:uniroute-llm-embed} on a validation set $\SVal = \{ ( \bx^{(j)}, \by^{(j)} ) \}_{j = 1}^{\NVal}$.
        \item \textbf{Route on new prompts.} 
        Given a new input prompt $\bx$, 
        pick the best test LLM via~\eqref{eqn:plugin-dynamic-routing}.
    \end{enumerate}
\end{tcolorbox}

\subsection{Representing an LLM via the \LossVectorSentenceCaps{}}
\label{sec:correctness_representation}

A key desideratum for
our LLM representation $\LLMEmbed$
is that 
$\LLMEmbed( h )^\top \LLMEmbed( h' )$ 
ought to be large for a pair $(h, h')$ of ``similar'' LLMs,
and small for a pair of ``dissimilar'' LLMs.
A reasonable definition of ``similar'' would thus enable the design of $\LLMEmbed$.
To this end,
we posit that two LLMs are ``similar'' if they 
\emph{have comparable performance on a set of representative prompts},
akin to proposals in~\citet{Thrush:2024,ZhuWuWen2024}.

Concretely, suppose 
that
we have access to a small 
(labelled)
validation set 
$\SVal = \{(\bx^{(j)},\by^{(j)})\}_{j=1}^{\NVal}$.
Further, suppose that 
\emph{any} LLM $h \in \msetAll$
--
\emph{including new LLMs unobserved during training}
--
can be evaluated
efficiently
on these prompts.
Then,
one may 
represent the LLM based on its
\emph{prediction error vector} on prompts from $\SVal$:
for suitable 
$F \colon \Real^{\NVal} \to \Real^K$
(e.g., a linear projection),
we choose:

\begin{tcolorbox}[
colback=gray!5!white,  %
colframe=gray!5!white, %
arc=3mm,                %
boxsep=0pt,          %
left=2pt,            %
right=2pt,           %
top=2pt,             %
bottom=2pt,          %
toptitle=0pt,        %
bottomtitle=1pt,     %
]
\begin{equation}
    \label{eqn:uniroute-llm-embed}
    \LLMEmbed( h ) = F\left( \begin{bmatrix} \1( \by^{(j)} \neq h( \bx^{(j)} ) ) \end{bmatrix}_{j \in [ \NVal ]} \right) \in \Real^K.
\end{equation}
\end{tcolorbox}
Interestingly, 
a special case of this
is the $K$-NN method~\eqref{eqn:knn-router} from~\citet{HuBieLi2024} 
applied to $\SVal$:
for 
$\LLMEmbed_{\text{knn}}( h ) = \begin{bmatrix} \1( \by^{(j)} \neq h( \bx^{(j)} ) ) \end{bmatrix}_{j \in [ \NVal ]} \in \Real^{\NVal}$ 
and 
$\QueryEmbed_{\text{knn}}( \bx ) \in \{ 0, 1 \}^{\NVal}$ indicating which validation samples are the $k$-nearest neighbours of $\bx$,~\eqref{eqn:gamma-uniroute} exactly reduces to~\eqref{eqn:knn-router}.
In general, however, it can prove useful to parameterise $F$ and \emph{learn} some compressed LLM representation in $K \ll \NVal$ dimensions.

We remark that~\citet{ZhuWuWen2024} also considered representing LLMs as feature vectors,
with the goal of enabling routing.
Crucially, however, their representations do \emph{not} enable generalisation to unseen LLMs,
and thus do not support dynamic routing;
\ifarxiv
cf.~\S\ref{sec:related} for more discussion.
\else
cf.~Appendix~\ref{app:related} for more discussion.
\fi

\subsection{Discussion: \ourMethod{} versus Standard Routing}

A central assumption in 
\ourMethod{}
is that 
for any new LLM,
one may efficiently compute
$\LLMEmbed( \cdot )$;
for $\LLMEmbed( \cdot )$ given by~\eqref{eqn:uniroute-llm-embed},
this 
in turn
assumes that
any new LLM can be efficiently evaluated on $S_{\rm val}$.
We stress some important points regarding this.
First,
the choice of prompts in $\SVal$ is of clear import.
In the simplest case, this may be a small subset of the training set $\STrain$.
More generally,
these could be hand curated based on domain knowledge,
or drawn from a standard benchmark suite
(which is often available for any new LLM).
Second,
we emphasise that
$\SVal$ is assumed to be of modest size 
(e.g., $\mathscr{O}( 10^{3} )$);
consequently,
performing inference for a new LLM on $\SVal$ is not prohibitive.
Third,
\ourMethod{} involves 
\emph{significantly less overhead} than 
na\"{i}ve router re-training.
Per~\S\ref{sec:setup},
re-training a router 
such as~\eqref{eqn:bert-router} on
$\SVal$ 
involves several steps of iterative training,
which 
for a constantly refreshing LLM pool
can impose a cumulatively onerous overhead.
Further, since $\SVal$ is of modest size,
re-training the router is susceptible to overfitting;
thus, \ourMethod{} can also yield better quality.

Interestingly,
the $K$-NN method
from~\citep{HuBieLi2024}
--
which, as noted in~\S\ref{sec:correctness_representation},
is a special case of \ourMethod{}
--
\emph{does} support new LLMs without re-training.
Indeed,
one may
readily compute 
$\gammaEstimator_{\text{kNN}}$ in~\eqref{eqn:knn-router}
based solely on the prediction error vector.
However, 
as
$\SVal$ is 
assumed to be
of modest size,
this approach 
may not generalise favourably;
indeed, even in moderate data regimes,~\citet{ZhuWuWen2024} observed that $K$-NN may underperform.
Further, it does not exploit any information from the 
(potentially large)
\emph{training} set $\STrain$.

These limitations notwithstanding,
$K$-NN has the appealing ability to exploit non-linear structure in the data.
We now explore a \emph{cluster-based} instantiation of \ourMethod{}
with a similar property.

\section{\ourMethod{} with Cluster-Based LLM Feature Representations}
\label{sec:proposal}

Building on the above,
we now consider
certain \emph{cluster-based} LLM representations, involving either unsupervised or supervised cluster assignments based on a large set of
training prompts.

\subsection{Representing an LLM via Per-Cluster Prediction Errors}
\label{sec:cluster_router}

We propose an instantiation of~\eqref{eqn:uniroute-llm-embed}
that represents any
LLM $h \in \msetAll$
through its average errors 
$\LLMClusterEmbedH \in [0,1]^K$
on $K > 1$ pre-defined \emph{clusters}. 
Similarly, we represent prompts via
their cluster membership
$\InputClusterEmbed \in \{ 0, 1 \}^K$.
This yields the following instantiation of~\eqref{eqn:gamma-uniroute}:
\begin{equation}
    \label{eqn:gamma-cluster}
    \gammaEstimator_{\cluster}(\bx, h) \defEq \InputClusterEmbed^\top \LLMClusterEmbedH.
\end{equation}
Intuitively,
$\gammaEstimator_{\cluster}(\bx, h )$ 
estimates the performance of a given LLM on 
a prompt $\bx$
by examining its performance on \emph{similar} prompts,
i.e.,
those belonging to the same cluster.

Na\"{i}vely, one may 
directly
cluster $\SVal$;
however,
this
is prone to overfitting,
since (by assumption) the set is of modest size.
Thus,
we instead cluster the prompts in the \emph{training} set $\STrain$.
We then
use this to
group the 
\emph{validation}
prompts into $K$ disjoint clusters, 
and compute per-cluster errors for a new LM using $\SVal$. %
Concretely, 
given a text embedder 
$\PretrainedQueryEmbed \colon \XCal \to \Real^{\PretrainedQueryEmbedDim}$,
we compute
$\InputClusterEmbed, \LLMClusterEmbedH$
via:
\begin{enumerate}[label=(\roman*),itemsep=1pt,topsep=0pt,leftmargin=16pt]
\item 
Cluster the \emph{training} set embeddings 
$\{\PretrainedQueryEmbed(\bx^{(i)})\}_{i=1}^{\NTrain}$
to construct $K$ non-overlapping clusters.
This yields a cluster assignment map
$\ClusterMembership \colon \mathscr{X}\to\{0,1\}^K$, 
where %
the $k$-th index is 1 when $\bx$ belongs to cluster $k$ (i.e., it is on average closest to samples in cluster $k$). This step does not require labels.
\item Assign each \emph{validation} set prompt to a cluster. 
Let $C_{k}\defeq\{(\bx, \by)\,:\,(\bx, \by) \in \SVal,\, \ClusterMembershipScalar{k}(\bx)=1\}$ be the subset of the validation set that belongs to cluster $k$.
\item For 
any LLM $h \in \msetAll$,
compute 
$\LLMClusterEmbedH \in [0,1]^{K}$ 
using its per-cluster validation errors: 
\begin{align}
\LLMClusterEmbedScalarHK \defeq\frac{1}{|C_{k}|}\sum_{(\bx,\by)\in C_{k}}\1\big[ \by \ne h(\bx) \big].
\label{eq:cluster_accs}
\end{align}
\end{enumerate}
Plugging these into~\eqref{eqn:gamma-cluster},
we may now approximate the expected loss for $\hNew^{(n)}$ on an input prompt $\bx$ using the average error of the LLM on the cluster the prompt is assigned to, and route via~\eqref{eqn:plugin-dynamic-routing}.

This cluster-based instantiation of \ourMethod{} 
can route with new LLMs in 
a highly efficient manner:
given any new test LLM, we simply need to estimate its
average per-cluster clusters for all validation prompts.
This
does \emph{not} require any expensive gradient updates,
and is a \emph{one-off cost}:
further routing
can operate entirely on this vector,
and is oblivious to any further changes in the LLM pool.

A natural choice of clustering algorithm in step (ii) is  $K$-means~\citep{Mac1967}, 
which returns a set of $K$ centroids 
and an assignment map $\ClusterMembership$ that assigns prompts 
to the cluster with the nearest centroid.
For $K = 1$, the router devolves to %
the \emph{ZeroRouter} from \citep{HuBieLi2024}
(see Appendix~\ref{app:pareto-random}).
Clearly, the 
practical utility of \ourMethod{}
depends on selection of $K$;
empirically, \ourMethod{} is reasonably robust to this parameter.
\ifarxiv
An illustration of our proposal is shown in \cref{fig:illus_cluster}.
\fi

\subsection{From Fixed to Learned Cluster Assignment Maps}
\label{sec:two_tower}

The above cluster-based router does not leverage the labels in $\STrain$.
We may exploit this information to further improve routing quality.
Specifically,
given the same clustering as above,
we propose to
\emph{learn} a
cluster assignment map
$\ClusterMembership( \cdot; \boldsymbol{\theta} ) \in [ 0, 1 ]^{K}$ 
parameterised by $\boldsymbol{\theta}$,
that can better map an input prompt to a distribution over clusters. 
Specifically, we parameterise
 $\ClusterMembershipScalar{k}( \boldsymbol{x}; \boldsymbol{\theta} ) \propto \exp\left(\theta_{k}^\top \PretrainedQueryEmbed( \boldsymbol{x} ) \right),$
for 
$\boldsymbol{\theta} \in \mathbb{R}^{K\times \PretrainedQueryEmbedDim}$
and text embedding $\PretrainedQueryEmbed$.
Analogous to~\eqref{eqn:gamma-cluster}, we have:
\begin{align*}
\gammaEstimator_{\cluster}( \bx, h; \boldsymbol{\theta} ) &= \InputLearnedClusterEmbed^\top \LLMClusterEmbedH, %
\end{align*}
where 
$\LLMClusterEmbedH$ 
denotes the per-cluster errors for the LM estimated from the validation set $\SVal$, as in~\eqref{eq:cluster_accs};
note that this does \emph{not} depend on $\boldsymbol{\theta}$.
To pick $\boldsymbol{\theta}$, we minimize the log loss on the training set 
$\STrain$
against 
the correctness labels for 
the training LMs $\mset_\tr$:%
\begin{align*}
    -\sum_{(\bx, \by) \in \STrain} &\sum_{h \in \mset_\tr}
        \1\big[ \by \ne h(\bx)\big] \cdot 
            \log\,\gammaEstimator_{\cluster}\left(
                    \bx, h; \boldsymbol{\theta} \right) 
            +
        \1\big[ \by = h(\bx)\big] \cdot 
            \log \left( 
                1 - \gammaEstimator_{\cluster}\left(
                    \bx, h; \boldsymbol{\theta} \right) 
                    \right).
\end{align*}

\subsection{Excess Risk Bound}
\label{sec:excess-risk}
We now present an excess risk bound for our cluster-based routing strategy.
Suppose we represent the underlying data distribution over $(\bx,\by)$ by a mixture of $K$ latent components:
$\mathbb{P}(\bx,\by) = \sum_{k=1}^{K}\pi_{k} \cdot \mathbb{P}(\bx,\by\,|\,z=k),$
where $z$
denotes a latent random variable that identifies the mixture
component and $\pi_{k}=\mathbb{P}(z=k)$. For a fixed $k$, we may denote the probability of incorrect predictions
for 
$h \in \mset$ 
conditioned on $z=k$ by:
$
\LLMEmbedScalar^*_{k}( h ) \defeq \mathbb{P}_{\bx, \by|z=k}\left[h(\bx)\ne\by\right].
$%
Then, cluster-based routing
seeks to mimic the following population rule: 
\asrnote{Is the assumption here that the *clustering* map is able to identify the true latent variable for each prompt/prompt?}
\begin{align}
    \tilde{r}^{*}(\bx, \msetNew)  &=
    \underset{n \in [ \numSetNew ]}{\argmin} \, \left[ \gamma^*_{\cluster}(\bx, \hNew^{(n)}) +\lambda \cdot c( \hNew^{(n)} ) \right];
    \label{eq:cluster-rule-population}
    \\
    \gamma_{\cluster}^*(\bx, \hNew^{(n)}) &= \sum_{k\in[K]}\mathbb{P}(z=k|\bx) \cdot \LLMEmbedScalar^*_{k}(\hNew^{(n)}).
    \nonumber
\end{align}

\begin{prop}
\label{prop:cluster-regret-bound}
Let $r^*$ 
be as per~\eqref{eq:opt_rule01}.
For any $\msetNew = \{ \hNew^{( n )} \}_{n \in [ N ]} \in \mathbb{H}$, 
$\hNew^{( n )} \in \msetNew$,
and
$\bx \in \XCal$, let:
\begin{align*}
    \Delta_{k}(\bx, \hNew^{( n )} ) \defEq \left|\mathbb{P}_{\by|\bx,z=k}\left[\hNew^{( n )}(\bx)\neq\by\right] \,-\, \LLMEmbedScalar^*_{k}( \hNew^{( n )} ) \right|.
\end{align*}
Let 
$R_{01}(r, \msetNew) \defEq \sum_{n} \mathbb{P} \left[ \hNew^{( n )}(\bx)\neq\by \land r(\bx, \msetNew)=m \right]$ 
denote the 0-1 risk. %
Then under a regularity condition on $\mathbb{P}$,
the difference in 0-1 risk between $\tilde{r}^*$ and $r^*$ is bounded by:
\begin{align*}
\mathbb{E}_{\msetNew}\left[R_{01}(\tilde{r}^*,\msetNew)\right] ~-~ \mathbb{E}_{\msetNew}\left[{R}_{01}(r^*,\msetNew)\right] & 
 \,\leq\,
\mathbb{E}_{\msetNew, \bx}\left[\max_{\hNew^{( n )} \in \msetNew, k\in[K]}\,\Delta_{k}(\bx, \hNew^{( n )} )\right].
\end{align*}
\end{prop}

This suggests that the quality gap 
between cluster-based routing in~\eqref{eq:cluster-rule-population} %
and the optimal rule in~\eqref{eq:opt_rule01} is bounded by the discrepancy between the per-cluster and per-prompt errors:
i.e.,
the gap between the LLM's 
error on a prompt
versus the 
average  constituent cluster error (see Appendix \ref{app:cluster-regret-bound} for proof).

\section{Related Work}
\label{sec:related}
\textbf{Model routing.}\
Model routing has emerged as a simple yet effective strategy to lower LLMs' inference cost~\citep{Hendy:2023,Narayanan:2023}.
Recent works have studied various strategies to learn a %
router,
including
training an explicit ``meta-model'' based on a neuronal network~\citep{Ding:2024,Sakota:2024,CheJiaLin2024,Aggarwal:2024},
$k$-nearest neighbours~\citep{HuBieLi2024,Shnitzer:2023,Stripelis:2024,Lee:2024},
matrix factorisation~\citep{OngAlmWu2024,ZhuWuWen2024,Li:2025},
and graph neural networks~\citep{Feng:2024}.
Works have also explored the role of supervision in training a router~\citep{LuYuaLin2024,Zhao:2024},
and enhancing router robustness~\citep{Dann:2024,Montreuil:2025,Shafran:2025}.
Typically,
the router orchestrates amongst multiple independent LLMs,
although it is also possible to route amongst \emph{implicit} sub-models in a larger model, such as those defined by a MatFormer~\citep{Devvrit:2024,Cai:2024}.

\textbf{Model cascading.}\
Cascading is a closely related technique for orchestrating amongst multiple models,
wherein one \emph{sequentially invokes} {each} model in order of cost.
One then uses statistics from the resulting model output (e.g., the confidence) to decide whether or not to proceed to the next costlier model.
Cascading has a long history in computer vision applications~\citep{VioJon2001,WanLuoCra2018,Swayamdipta:2018,Rawat:2021,WanKonChr2022,Kag:2023,Jitkrittum:2023}.
Recently, cascades have also been successfully proven in the case of LLMs~\citep{Varshney:2022,CheZahZou2023,GupNarJit2024,Yue:2024,Chen:2024,Nie:2024,Chuang:2025}.

\ifarxiv

\setlength{\tabcolsep}{0.2em} %
\renewcommand{\arraystretch}{1.2}  %

\begin{table}[!t]
    \begin{centering}
    \resizebox{\linewidth}{!}{
    \begin{tabular}{>{\raggedright}p{3cm}>{\centering}m{1.5cm}>{\centering}m{4cm}>{\centering}m{1.5cm}>{\centering}m{2cm}>{\centering}m{1.5cm}>{\centering}m{3cm}}
    \toprule 
    Routing approach & Candidate LLMs & Training signals & Works without task labels & Adaptive computation & Unseen models at test time & Reference\tabularnewline
    \midrule
    Smoothie  & Any & Query embeddings. No label required. & \boldcheck &  \crossmark &  \crossmark & \citet{GuhCheCho2024}\tabularnewline
    Cascading with token-level features  & 2 & Pointwise evaluation. & \boldcheck & \boldcheck &  \crossmark & \citet{GupNarJit2024}\tabularnewline
    Summon the titans  & 2 & Annotations from a teacher model. & \boldcheck & \boldcheck  &  \crossmark & \citet{Rawat:2021} \tabularnewline
    RouteLLM  & 2 & Pairwise comparison metrics. & \boldcheck & \boldcheck  &  \crossmark & \citet{OngAlmWu2024}\tabularnewline
    $K$-NN router & Any & Pointwise evaluation, query embeddings.  & \boldcheck & \boldcheck  &  $\triangle$ & \citet{HuBieLi2024,Shnitzer:2023}\tabularnewline
    GraphRouter & Any &  Task information &  \crossmark & \boldcheck  & \boldcheck  & \citet{Feng:2024}\tabularnewline
    \midrule 
    Our proposal & Any & Pointwise evaluation, query clusters & \boldcheck & \boldcheck  & \boldcheck & This work\tabularnewline
    \bottomrule
    \end{tabular}
    }
    \par\end{centering}
    \caption{A qualitative comparison of recently proposed model routing approaches.
    Adaptive computation refers to the ability to trade quality for a
    reduced inference cost. 
    $\triangle$: The $K$-NN approach considered in \citet{HuBieLi2024,Shnitzer:2023} is for a fixed pool of LLMs.
    However, the approach can be straightforwardly extended to support unseen models at test time (i.e., dynamic pool). 
    }
\end{table}

\textbf{Selective classification and learning to defer.}\ 
The formal underpinnings of routing and cascading can be traced to
three closely related literatures:
learning to reject~\citep{chow1970optimum,Bartlett:2008,Cortes:2016}, 
selective classification~\citep{Geifman:2019,NarMenJit2024,NarMenJit2024a},
and learning to defer to an expert~\citep{Madras:2018,SanErdKon2023}.
Following pioneering studies of~\citet{Trapeznikov:2013,Bolukbasi:2017,MozSon2020},
a series of works have studied the routing and cascading problem through these lenses~\citep{NarJitMen2022,MaoMohMoh2024,Mao2024,MaoMohZho2024,MaoMohZho2024a}.

\textbf{Model fusion}
Model routing may be contrast to model \emph{fusion}, 
where the primary goal is to leverage multiple models to improve \emph{quality}, 
potentially at the expense of \emph{efficiency}.
This can involve
invoking multiple models prior to generating an output~\citep{Ravaut:2022,Jiang:2023,Guha:2024,Wang:2024,Hu:2024b},
or producing a single fused model~\citep{Singh:2020}.

\textbf{Mixture of experts (MoE).}\
Classically, MoE models focused on learning parameters for independent models, along with a suitable routing rule~\citep{Jacobs:1991,Jordan:1993}.
These have proven an plausible alternative to model specialisation~\citep{Jang:2023,Douillard:2024}.
Such models are typically of the same size; thus, cost considerations do not factor into the router design.
More recently, MoEs have focussed on \emph{sub}-models within a single larger model, e.g., a Transformer~\citep{Fedus:2022,Zhou:2022}.

\textbf{Early-exiting.}\
Early-exiting enables adaptive computation within a \emph{single} neural model,
by allowing computation to terminate at some intermediate layer~\citep{TeeMcDKun2016,ScaScaBac2020,Zhou:2020}. 
Often, the termination decision is based on a simple model confidence (akin to simple model cascading),
but learning approaches have also been considered~\citep{XinNogYu2020,SchFisGup2022}.

\textbf{Speculative decoding.}\
Speculative decoding is another technique that leverages two models to speed up inference, by using the smaller model to draft tokens and having the larger model verify them~\citep{Stern:2018,chen2023accelerating,leviathan2023fast,Tran-Thien_2023,Sun:2023,Zhou:2024,Cai:2024b,Li:2024,Li:2024b}.
Recent
works have studied approaches to combine
speculative decoding
with
early-exits~\citep{Elhoushi_2024} 
and
cascades~\citep{Narasimhan:2024}.

\else

Our solution allows a trained router to make use of these test-time LLMs without re-training.
This is unlike most, though not all existing solutions to the problem;
indeed, recent works~\citet{Feng:2024,Zhao:2024,Li:2025} have also proposed solutions to dynamic routing.
Compared to these works, we provide a formally-grounded approach that 
avoids dependence on exogenous \emph{task label} information for prompts, and relies on established statistical learning primitives; 
we defer a complete discussion to Appendix~\ref{sec:discussion_related}.
\fi

\section{Experiments}
\label{sec:experiments}
We demonstrate the effectiveness of 
\ourMethod{}
in the setting
of observing \textbf{new LLMs at test time}
on 
EmbedLLM \citep{ZhuWuWen2024}, 
RouterBench \citep{HuBieLi2024},
a Math+Code dataset from \citet{Dekoninck:2025} (containing a subset of Minerva Math and LiveCodeBench \citep{lewkowycz2022solving,jain2024livecodebench}),  
SPROUT o3-mini \citep{Somerstep:2025}, and 
Chatbot Arena \citep{ZheChiShe2023}.

\newcommand{\sig}{\hla{$^{*}$}}
\newcommand{\nosig}{\phantom{\sig}}

\begin{figure*}[!t]
  \begin{minipage}[b]{.98\linewidth}
  \centering
  \resizebox{\columnwidth}{!}{
    \begin{tabular}{lrrrrrrrrrrrr@{}}
    \toprule
    \multirow{2}{*}{
    \diaghead(-3,1){justpadspaceeee}{Method}{Dataset}
    } 
    & \multicolumn{3}{c}{EmbedLLM}  
    & \multicolumn{3}{c}{RouterBench} 
    & \multicolumn{3}{c}{Math+Code} 
    & \multicolumn{3}{c}{SPROUT o3-mini} 
    \\
    \cmidrule(lr){2-4} \cmidrule(lr){5-7} \cmidrule(lr){8-10} \cmidrule(lr){11-13}
     & Area (50\%) $\uparrow$ & Area  $\uparrow$ & QNC $\downarrow$ 
     & Area (50\%) $\uparrow$ & Area  $\uparrow$ & QNC $\downarrow$ 
     & Area (50\%) $\uparrow$ & Area  $\uparrow$ & QNC $\downarrow$ 
     & Area (50\%) $\uparrow$ & Area  $\uparrow$ & QNC $\downarrow$ 
     \\
    \midrule
    ZeroRouter~\citep{HuBieLi2024}
       & .285\sig & .607\sig & 87.5\%\sig
       & .320\sig & .689\sig & 99.9\%\nosig
      & .193\sig & .395\sig & 82.8\%\sig
      & .404\sig & .820\sig & 100.0\%\sig
      \\
      
    $K$-NN~\citet{HuBieLi2024,Shnitzer:2023}
       & .298\sig & .636\sig & 46.1\%\sig
        & .328\sig & .707\sig & 99.7\%\nosig 
        & .237\sig & .487\sig &	25.7\%\nosig
        & .418\sig & .844\sig & 29.6\%\sig
        \\
     \ourMethod{} ($K$-Means) 
        & .307\sig & .648\sig & 33.9\%\nosig
        & \best{.332}\nosig & \best{.712}\nosig & \best{99.4}\%\nosig &
        \best{.238}\nosig &	\best{.490}\nosig	& \best{25.7}\%\nosig	  & \best{.421}\nosig & \best{.850}\nosig & \best{19.6}\%\nosig
        \\
    \ourMethod{} (LearnedMap)
        & \best{.308}\nosig & \best{.651}\nosig & \best{33.2\%}\nosig
        & .331\nosig & .711\nosig & 99.6\%\nosig &
        ---\nosig & ---\nosig & ---\nosig
        & .420\nosig & .846\nosig & 23.4\%\nosig
        \\
    \midrule
       MLP (Clairvoyant)%
        & .314\nosig & .664\nosig & 26.9\%\nosig
        & .339\nosig & .723\nosig & 95.2\%\nosig 
        & .242\nosig & .500\nosig & 25.1\%\nosig
        & .427\nosig & .859\nosig & 4.5\%\nosig
        \\
    \bottomrule
    \end{tabular}
   }
  \end{minipage}
  \begin{minipage}[b]{.98\linewidth}
  \vspace{5pt}
         \begin{subfigure}{0.245\textwidth}
          \includegraphics[scale=0.22]{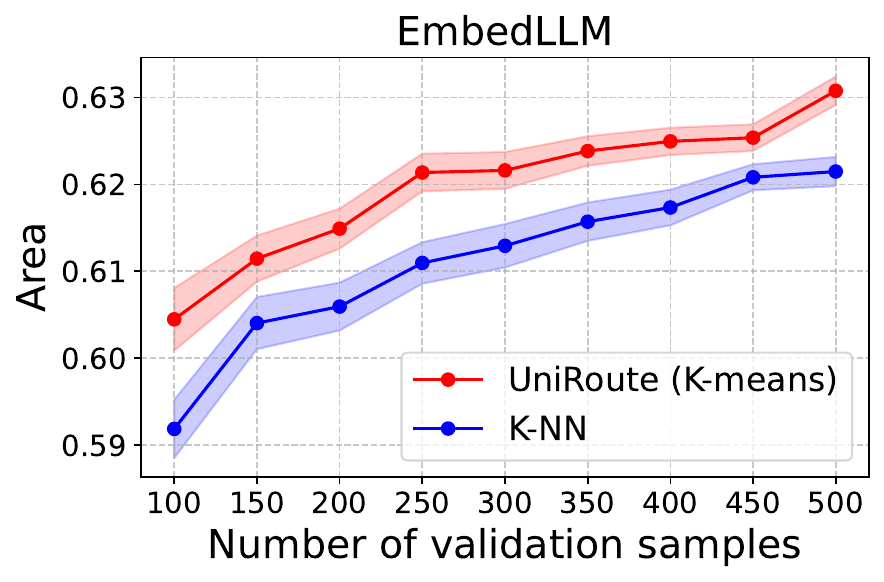}
        \end{subfigure}
        \begin{subfigure}{0.245\textwidth}
        \includegraphics[scale=0.22]{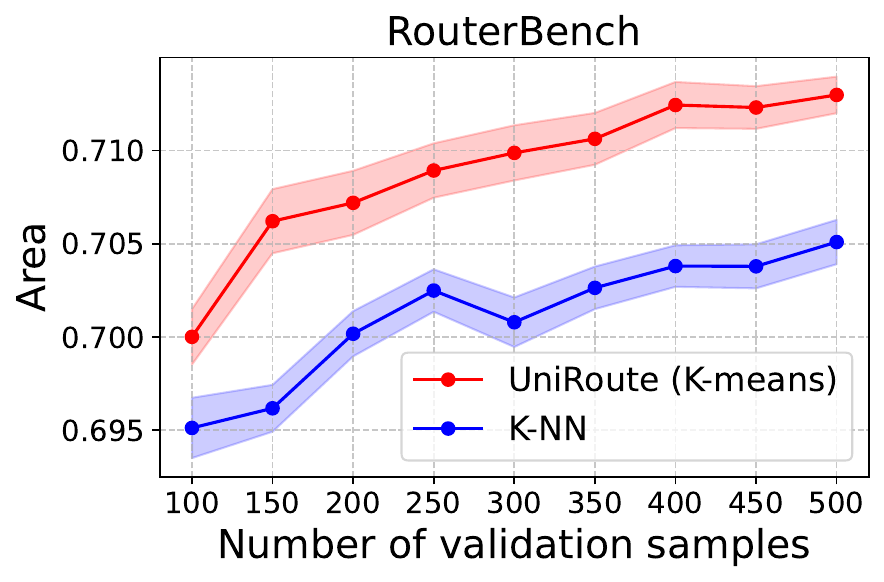}
        \end{subfigure}
        \begin{subfigure}{0.245\textwidth}
        \includegraphics[scale=0.22]{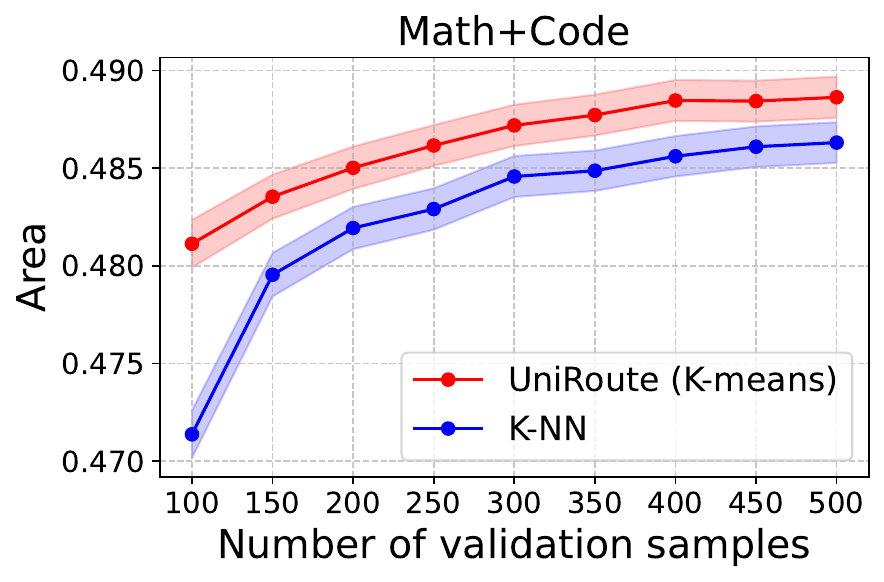}
        \end{subfigure}
        \begin{subfigure}{0.245\textwidth}
        \includegraphics[scale=0.22]{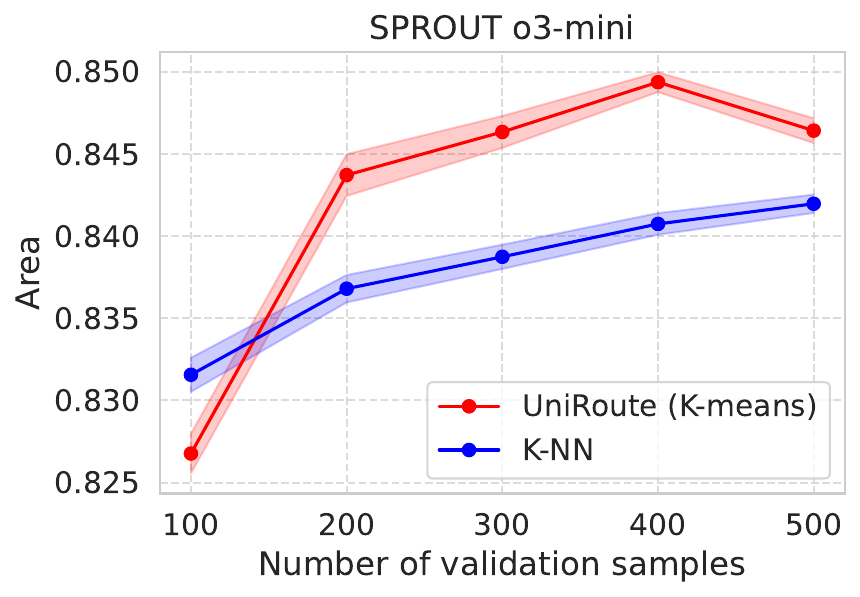}
        \end{subfigure}
  \end{minipage}
  \caption{
    {\textbf{Top}}: %
    We report the \emph{area} under the deferral curve (up to $50\%$ and $100\%$ cost), and the Quality-Neutral Cost (QNC), i.e., the minimum relative cost to achieve the same performance as the most accurate LLM.  For Math+Code, we do not have training LLMs; so we do not report results for \ourMethod{} (LearnedMap). $\sig$ indicates the method is statistically significantly worse than \ourMethod{} (LearnedMap) at significance level $\alpha=0.01$ (we compare against K-means for Math+Code). \textbf{MLP (Clairvoyant) is an \textit{oracle}} that uses the test LLMs for training (provides a performance \emph{upper bound}).  %
    {\textbf{Bottom}}: Areas under the deferral curve ($\uparrow$) with $96\%$ CI on \textbf{unseen test LLMs} for varying number of validation samples.  \ourMethod{} ($K$-means) consistently outperforms $K$-NN for small sample sizes.
    }
\vspace{-12pt}
\label{fig:three_experiments}
\end{figure*}

\ifarxiv
\subsection{Experimental Setup}
\else
\underline{\textbf{Experimental setup}}.
\fi
We first describe the 
\textbf{LLM pool construction.}
With EmbedLLM, RouterBench, and SPROUT o3-mini, we partition the set of LLMs  into two disjoint sets: training models ($\mset_\mathrm{tr}$ in \S\ref{sec:proposal}) and testing models ($\mset_\mathrm{te}$). 
For EmbedLLM ($112$ LLMs) and SPROUT o3-mini (15 LLMs), we use a random subset of $\sfrac{2}{3}$ for training and $\sfrac{1}{3}$ for testing. For  RouterBench (11 LLMs in total), we use a random 50\% 
for training and the rest for testing. For Math+Code, we use all 4 LLMs for testing; consequently, the training data is unlabeled.

\wjtext{For each of the 400 independent trials}, we randomly split examples into 60\%/10\%/30\% for training, validation, and testing %
(for RouterBench, and SPROUT o3-mini we use 1\% for validation). 
The training portion is for training a router, and only has correctness labels of training models.
The validation split is the small dataset used to represent each test-time LLM as a feature vector (see \S\ref{sec:cluster_router}). 
All baselines are evaluated on the test examples and \emph{only} on the test LLMs. 
We use Gecko 1B~\citep{LeeDaiRen2024} to produce a $768$-dimensional prompt embedding where required.

\textbf{Per-example metrics.}\
All datasets considered in the main text use binary accuracy as the 
the evaluation metric.
Thus, all 
methods
rely on the deferral rule described in \eqref{eq:opt_rule01}.

\textbf{Baselines.}\ 
We reiterate that, with the notable exception of $K$-NN \citep{HuBieLi2024}, \textit{\textbf{most existing routers in the literature are inadmissible}} in a settingwith a \emph{dynamic} pool of LLMs. This is true of the multi-layer perceptron (MLP) \citep{HuBieLi2024}, matrix factorization \citep{OngAlmWu2024,ZhuWuWen2024} and BERT \citep{OngAlmWu2024,Ding:2024} routers, which have a \emph{fixed} number of outputs, one per training LLM, and are thus inherently tied to those LLMs.
Nonetheless, we include some of these methods as
an \emph{oracle} %
(by assuming the set of LLMs is fully known) to estimate an \emph{upper bound} on router performance on unseen LLMs.
The baselines we consider are:
\begin{itemize}[itemsep=0pt,topsep=0pt,leftmargin=16pt]%
    \item \textit{ZeroRouter} \citet{HuBieLi2024}. 
    A \textit{random router} that randomizes between two LLMs, where the LLMs and mixing coefficients are chosen to maximize the expected quality on the validation sample, while satisfying the budget constraint 
    (details in Appendix \ref{app:pareto-random}).

    \item \textit{$K$-NN} \citep{HuBieLi2024,Shnitzer:2023}. A special case of \ourMethod{} (see \S\ref{sec:correctness_representation}) that for each prompt, looks up the $K$ nearest prompts in the validation set in the space of prompt embeddings, computes $\gammaEstimator$ for each test LLM using \eqref{eqn:knn-router} (with the 0-1 loss), and routes via  \eqref{eqn:post-hoc}.
    
    \item \textit{MLP (Clairvoyant upper bound)}. An MLP router with one output for each train and test LLM, trained on prompt embeddings to estimate $\gammaEstimator$. For training, we use the combined training and validation set, annotated with correctness labels from both the train and test LLMs. %
    This {oracle} baseline provides an estimate %
     of the \emph{performance achievable when all LLMs are observed}.

\end{itemize}

We compare them to our \ourMethod{} cluster-based routing method using both (i) unsupervised $K$-means  for clustering (\S\ref{sec:cluster_router}), and (ii) the supervised learned cluster assignment map (\S\ref{sec:two_tower}).  
In \Cref{app:expts-additional}, we also include as a baseline \emph{Clairvoyant version of the matrix factorization} router~\citep{OngAlmWu2024,ZhuWuWen2024}. %

\begin{wrapfigure}{r}{0.45\textwidth}
  \begin{center}
  \vspace{-5pt} 
    \includegraphics[width=0.45\textwidth]{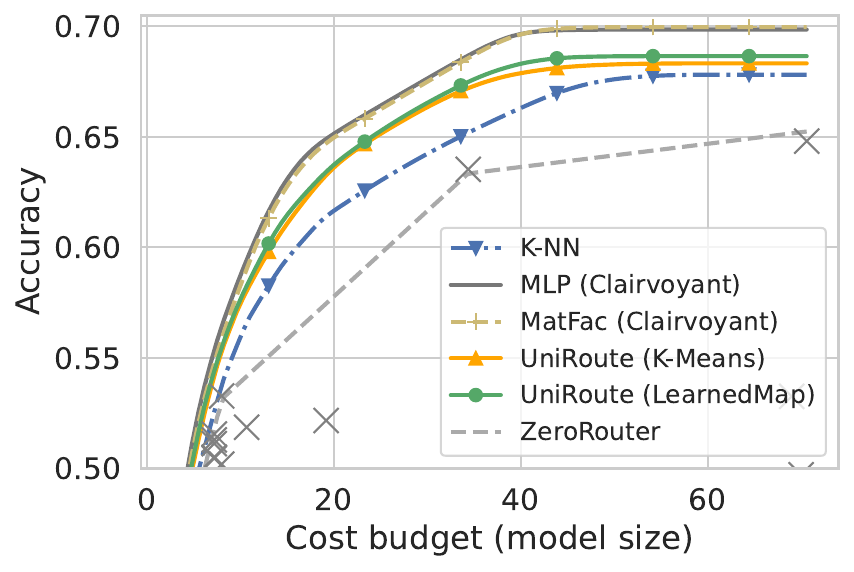}
  \end{center}
  \vspace{-5pt}
  \caption{Deferral curves on EmbedLLM.}
  \vspace{-5pt}
  \label{fig:deferral-curve-embedllm}
\end{wrapfigure}

\textbf{Evaluation.}\ We evaluate each method with a \emph{deferral curve}
(Appendix~\ref{sec:deferral-curve};~\citep{Jitkrittum:2023,WanAugRus2024,HuBieLi2024}),
which plots the 
trade-off of
average quality against cost.
The trade-off is realized by varying $\lambda_\mathfrak{H}$ in the routing rule in \eqref{eq:opt_rule}.
For EmbedLLM, we use the number of parameters of the LLM as the cost of processing one prompt. 
This cost is a proxy for the amount of computation required to call each LLM.
For RouterBench and SPROUT o3-mini, we plot LLMs' API calling costs (in USD) as available in the dataset.

\textbf{Hyper-parameter tuning.}\ We apply the following procedure to pick $K$ for $K$-NN, $K$-means, %
and the Learned cluster map: %
for each parameter, we represent the training LLM using correctness labels in the training set, evaluate the routing rule for the training LLMs on the validation set, and measure the area under the deferral curve. This evaluation metric can be seen as the average improvement in quality per unit cost (analogous to the AIQ metric used in \citet{HuBieLi2024}). The parameter with the maximum area is then chosen. See Appendix \ref{sec:validate_k} for details.

\ifarxiv
\subsection{Experimental Results}
\else
\underline{\textbf{Experimental results}}.
\fi
We present deferral curves for different methods on EmbedLLM in \cref{fig:deferral-curve-embedllm}, and on other datasets in Appendix \ref{app:deferral}. 
Each isolated  point \textcolor{gray}{$\times$} represents the cost and average test accuracy of one testing LLM.
In the table, we report three evaluation metrics for each method: (i) the area under the deferral curve (Area); (ii) the area up to 50\% cost; and (iii) the quality-neutral cost (QNC) or the minimum relative cost needed to achieve the same performance as the most accurate testing LLM. %
Note that the QNC is analogous to the  call-performance threshold %
reported in \citet{OngAlmWu2024}.  The table in Figure \ref{fig:three_experiments} summarizes these metrics for all four datasets.

\textbf{\ourMethod{} enables generalisation under dynamic LLM pools}.\ \ourMethod{} yields competitive quality-cost trade-offs, with the gains over $K$-NN being \emph{statistically significant}.
In particular, on EmbedLLM
--
featuring \textbf{$> 30$ unseen LLMs}
-- 
\harichange{we provide compelling gains over $K$-NN on all metrics.} 
Further,
on all four datasets, 
we 
consistently outperform
ZeroRouter, which was noted as a strong baseline in~\citet{HuBieLi2024}.
Appendix~\ref{sec:chatbot_exp_details} further shows \ourMethod{} 
is effective when the LLM representations are constructed from Chatbot Arena, and evaluated  on EmbedLLM.

\textbf{\ourMethod{} is effective even with a small validation sample}.\
\harichange{We show in Figure \ref{fig:three_experiments} (bottom), that our proposal is often significantly better than $K$-NN across a range of validation sample sizes.} 
One potential reason for this is
the requirement in $K$-NN that only the retrieved neighbors from the validation set can be used to estimate test models' performance. It is hence unable to exploit the training set in any way. In contrast, our methods are able to exploit the training data either in an unsupervised ($K$-means)
or
supervised (Learned cluster map)
manner.

\textbf{\ourMethod{} is robust to choice of clusters $K$}.
\Cref{app:expts-additional} shows that
in general,
\ourMethod{} is effective for several different $K$ values,
and
our hyperparameter selector picks an effective $K$.

\textbf{\ourMethod{} maintains generalisation under static LLM pools}.\ 
While the 
dynamic pool setting is the focal point of our work,
we show in Appendix~\ref{app:static-pool}
that even in the \emph{static} LLM pool setting,
\ourMethod{} typically performs comparable to most baselines.

\textbf{Qualitative analysis of $\LLMEmbed( h )$ embeddings}.\
Appendix~\ref{sec:embed-viz} presents further qualitative analysis of the LLM embeddings deriving from the \LossVector{} representation.
These show that a largely intuitive grouping of ``similar'' LLMs (e.g., coding-specialists) in the embedding space.

\section{Conclusion and Future Work}
We present principled strategies for routing amongst multiple unseen test-time LLMs,
by leveraging
a \emph{\LossVector{}} LLM representation.
An interesting future direction is to 
enhance routing robustness to prompt distribution shifts,
such as by allowing the set of representative prompts $S_{\rm val}$ to dynamically vary. 
Such a routing system will further reduce the need for frequent router re-training.
\asrnote{Let's add an impact statement as well?}

\bibliographystyle{plainnat}
\bibliography{refs}

\clearpage
\appendix

\addcontentsline{toc}{section}{Appendix} %
\part{Appendix} %
\parttoc %

\raggedbottom 
\section{Limitations}
\label{sec:limitations}

Our work has some limitations that would be worthy subjects for future research.
First, in the fully static LLM pool setting, our proposed method is not \emph{guaranteed} to recover the performance of existing methods such as the linear router~\eqref{eqn:bert-router} (owing to a dependence on the validation sample, which may be a subset of the training data);
designing \emph{hybrid} static-dynamic routers would be of interest.

Second, while we proposed certain natural cluster-based routers, there could be other instances of \ourMethod{} that are also worth exploring.
A more systematic study of this design space would be worthwhile.

\section{Societal Impact}
\label{sec:impact}

Our primary contribution is a mathematical formalism for routing amongst multiple dynamic LLMs, with concrete instantiations based on prompt clustering.
These are coupled with empirical results on public benchmarks, typically involving LLM accuracy as the primary metric.

We do not foresee immediate negative societal impact from our mathematical framework.
While well-studied in the literature, the underlying routing problem itself \emph{could} lend itself to some undesirable outcomes;
e.g., a router may overly favour models that produce outputs that are systematically biased on certain data sub-groups.
We believe that our framework is amenable to such constraints,
and believe that
accounting for these is a worthy subject of future research.

\section{Proofs of Results in Main Body}
\label{sec:proof}

In what follows, we use $r_\mset(\bx)$ and $r(\bx, \mset)$ interchangeably.
As a shorthand, we define
\begin{align*}
S(r) & \defeq\mathbb{E}_{(\bx,\by,\mset)}\left[\sum_{m\in[|\mset|]}\1(r(\bx,\mset)=m)\cdot\ell(\bx,\by,h^{(m)})\right], \\
C(r) & \defeq\mathbb{E}_{(\bx,\mset)}\left[\sum_{m\in[|\mset|]}\1(r(\bx,\mset)=m)\cdot c(h^{(m)})\right].
\end{align*}
The constrained optimization problem in \eqref{eq:risk} is equivalent to
\begin{align}
 & \min_{r\in\mathscr{R}}S(r)\text{ subject to }C(r)\le B.
 \label{eq:shortcon}
\end{align}
It will be useful to consider the following related unconstrained optimization objective:
\begin{align}
L(r,\lambda) & =S(r)+\lambda\cdot C(r).
\label{eq:shortuncon}
\end{align}
For any $\lambda\ge0$, let $r_{\lambda}^{*}$ be the minimiser of $L(r,\lambda)$.

\subsection{Intermediate Results}
We first provide results that will be useful for providing the main claims in \S\ref{sec:proof_opt_rule} and \S\ref{app:cluster-regret-bound}.

\begin{lem}
\label{lem:opt_uncon}Given $\lambda\ge0$, the optimal solution $r_{\lambda}^{*}\in\arg\min_{r}L(r,\lambda)$
to the unconstrained problem in (\ref{eq:shortuncon}) is
\begin{align*}
r_{\lambda}^{*}(\bx,\mset) & \in\argmin_{m\in[|\mset|]}\mathbb{E}_{\by|\bx}\left[\ell(\bx,\by,h^{(m)})\right]+\lambda\cdot c(h^{(m)}).
\end{align*}
\end{lem}

\begin{proof}
Starting with the definition of $L$,
\begin{align*}
L(r,\lambda) & =S(r)+\lambda\cdot C(r)\\
 & =\mathbb{E}_{(\bx,\by,\mset)}\left[\sum_{m\in[|\mathscr{H}|]}\1(r(\bx,\mathscr{H})=m)\cdot\ell(\bx,\by,h^{(m)})\right]+\lambda\cdot\mathbb{E}_{(\bx,\mathscr{H})}\left[\sum_{m\in[|\mathscr{H}|]}\1(r(\bx,\mathscr{H})=m)\cdot c(h^{(m)})\right]\\
 & \stackrel{(a)}{=}\mathbb{E}_{\mset}\mathbb{E}_{\bx}\mathbb{E}_{\by|\bx}\left[\sum_{m\in[|\mathscr{H}|]}\1(r(\bx,\mathscr{H})=m)\cdot\left\{ \ell(\bx,\by,h^{(m)})+\lambda\cdot c(h^{(m)})\right\} \bigg|\thinspace\bx\right]\\
 & =\mathbb{E}_{\mset}\mathbb{E}_{\bx}\bigg[\underbrace{\sum_{m\in[|\mathscr{H}|]}\1(r(\bx,\mathscr{H})=m)\cdot\left\{ \mathbb{E}_{\by|\bx}\left[\ell(\bx,\by,h^{(m)})\big|\thinspace\bx\right]+\lambda\cdot c(h^{(m)})\right\} }_{\mathcal{L}_{\mset,\bx}}\bigg],
\end{align*}
where $(a)$ uses the fact that the draw of $\mset$ is independent
of the draw of $(\bx,\by)$. The last line makes it clear that for
any fixed $\mset\sim\mathfrak{H}$, and any fixed $\bx$, to minimize
the overall loss, the router ought to route to the model that has
the lowest cost-adjusted loss $\mathcal{L}_{\mset,\bx}$. Thus, 
\[
r^{*}(\bx,\mset)\in\argmin_{m\in[|\mset|]}\mathbb{E}_{\by|\bx}\left[\ell(\bx,\by,h^{(m)})\right]+\lambda\cdot c(h^{(m)}).
\]
\end{proof}

\begin{lem}
\label{lem:lagrange_minimiser}For any $\lambda\ge0$, let $r_{\lambda}^{*}$
be the minimiser of $L(r,\lambda)$. Then, for any $r$, if $C(r_{\lambda}^{*})\ge C(r)$,
then $S(r_{\lambda}^{*})\le S(r)$.
\end{lem}

\begin{proof}
Since $r_{\lambda}^{*}$ minimises $L(r,\lambda)$, for any $r$,
we have $L(r_{\lambda}^{*},\lambda)\le L(r,\lambda)$. That is, 
\begin{align}
 & S(r_{\lambda}^{*})+\lambda\cdot C(r_{\lambda}^{*})\le S(r)+\lambda\cdot C(r)\label{eq:lagrange_min_implies}\\
\iff & S(r_{\lambda}^{*})\le S(r)+\lambda\cdot\left(C(r)-C(r_{\lambda}^{*})\right).\nonumber 
\end{align}
Since $\lambda\ge0$, it follows that $S(r_{\lambda}^{*})\le S(r)$
by the assumption that $C(r)-C(r_{\lambda}^{*})\le0$.
\end{proof}

\begin{lem}
\label{lem:lambda_penalizes_cost}Let $r_{\lambda}^{*}$
be the minimiser of $L(r,\lambda)$. For $\lambda,\lambda'\ge0$, $\lambda\ge\lambda'$
if and only if $C(r_{\lambda}^{*})\le C(r_{\lambda'}^{*})$. In other
words, $\lambda\mapsto C(r_{\lambda}^{*})$ is a non-increasing function.
Hence, $\sup_{\lambda\ge0}C(r_{\lambda}^{*})=C(r_{0}^{*})$.
\end{lem}

\begin{proof}
Since $r_{\lambda}^{*}\in\arg\min_{r}L(r,\lambda),$ by definition,
we have 
\begin{align*}
S(r_{\lambda}^{*})+\lambda\cdot C(r_{\lambda}^{*}) & \le S(r)+\lambda\cdot C(r),
\end{align*}
for any $r$, including $r_{\lambda'}^{*}$. This means
\begin{align*}
S(r_{\lambda}^{*})+\lambda\cdot C(r_{\lambda}^{*}) & \le S(r_{\lambda'}^{*})+\lambda\cdot C(r_{\lambda'}^{*}).
\end{align*}
In a symmetric manner,
\begin{align*}
S(r_{\lambda'}^{*})+\lambda'\cdot C(r_{\lambda'}^{*}) & \le S(r_{\lambda}^{*})+\lambda'\cdot C(r_{\lambda}^{*}).
\end{align*}
Adding the above two inequalities, we have
\begin{align*}
S(r_{\lambda}^{*})+\lambda\cdot C(r_{\lambda}^{*})+S(r_{\lambda'}^{*})+\lambda'\cdot C(r_{\lambda'}^{*}) & \le S(r_{\lambda'}^{*})+\lambda\cdot C(r_{\lambda'}^{*})+S(r_{\lambda}^{*})+\lambda'\cdot C(r_{\lambda}^{*})\\
\implies\lambda\cdot C(r_{\lambda}^{*})+\lambda'\cdot C(r_{\lambda'}^{*}) & \le\lambda\cdot C(r_{\lambda'}^{*})+\lambda'\cdot C(r_{\lambda}^{*})\\
\implies\lambda\cdot\left(C(r_{\lambda}^{*})-C(r_{\lambda'}^{*})\right)+\lambda'\cdot\left(C(r_{\lambda'}^{*})-C(r_{\lambda}^{*})\right) & \le0\\
\implies\left(\lambda-\lambda'\right)\left(C(r_{\lambda}^{*})-C(r_{\lambda'}^{*})\right) & \le0.
\end{align*}
The last inequality above implies that $\lambda\ge\lambda'$ if and
only if $C(r_{\lambda}^{*})\le C(r_{\lambda'}^{*})$.
\end{proof}

\begin{lem}
\label{lem:routed_model_changes}Let $\lambda,\lambda'\ge0$ such
that $\lambda\le\lambda'$. Given $\bx\in\mathscr{X}$ and $\mset\in\mathbb{H}$,
if $r_{\lambda}^{*}(\bx,\mset)\neq r_{\lambda'}^{*}(\bx,\mset)$,
then there exists a model pair $h^{(m)},h^{(k)}\in\mset$ with $h^{(m)}\neq h^{(k)}$
such that 
\begin{align*}
 & \lambda\le\frac{\text{\ensuremath{\mathbb{E}_{\by|\bx}\left[\ell(\bx,\by,h^{(k)})\right]}}-\mathbb{E}_{\by|\bx}\left[\ell(\bx,\by,h^{(m)})\right]}{c(h^{(m)})-c(h^{(k)})}\le\lambda'.
\end{align*}
\end{lem}

\begin{proof}
Given $\mset,$ define $A_{m}(\lambda,\bx)\defeq\mathbb{E}_{\by|\bx}\left[\ell(\bx,\by,h^{(m)})\right]+\lambda\cdot c(h^{(m)})$.
Recall that 
\[
r_{\lambda}^{*}(\bx,\mset)=\argmin_{m\in[|\mset|]}A_{m}(\lambda,\bx).
\]
Observe that given $(\bx,\mset)$, for each $m\in[|\mset|]$, $A_{m}(\lambda,\bx)$
is a one-dimensional affine function of $\lambda$. So, $r_{\lambda}^{*}(\bx,\mset)$
is the index of the affine function that gives the lowest value when
evaluated at $\lambda$. Fix $\bx$. Since $\lambda\mapsto A_{m}(\lambda,\bx)$
is continuous, for any $m$, if $r_{\lambda}^{*}(\bx,\mset)\neq r_{\lambda'}^{*}(\bx,\mset)$,
it must mean that the index of the lowest affine function changes
as we move from $\lambda$ to $\lambda'$. This implies that there
is an index pair $m$ and $k$ (with $m\neq k$) for which $\lambda\mapsto A_{m}(\lambda,\bx)$
and $\lambda\mapsto A_{k}(\lambda,\bx)$ cross, as we move from $\lambda$
to $\lambda'$. The cross point is precisely at $\bar{\lambda}$ such
that
\begin{align*}
 & A_{m}(\bar{\lambda},\bx)=A_{k}(\bar{\lambda},\bx)\\
\iff & \mathbb{E}_{\by|\bx}\left[\ell(\bx,\by,h^{(m)})\right]+\bar{\lambda}\cdot c(h^{(m)})=\mathbb{E}_{\by|\bx}\left[\ell(\bx,\by,h^{(k)})\right]+\bar{\lambda}\cdot c(h^{(k)})\\
\iff\bar{\lambda}= & \frac{\text{\ensuremath{\mathbb{E}_{\by|\bx}\left[\ell(\bx,\by,h^{(k)})\right]}}-\mathbb{E}_{\by|\bx}\left[\ell(\bx,\by,h^{(m)})\right]}{c(h^{(m)})-c(h^{(k)})}.
\end{align*}
\end{proof}

\begin{lem}
\label{lem:cost_continuous}Assume $\bx$ is a continuous random vector.
Then, $\lambda\mapsto C(r_{\lambda}^{*})$ is continuous.
\end{lem}

\begin{proof}
We will show that for any $\Delta\lambda\in\mathbb{R}$, $\lim_{\Delta\lambda\to0}|C(r_{\lambda+\Delta\lambda}^{*})-C(r_{\lambda}^{*})|=0$.
Observe that
\begin{align*}
|C(r_{\lambda+\Delta\lambda}^{*})-C(r_{\lambda}^{*})| & =\left|\mathbb{E}_{(\bx,\mset)}\left[c\left(h^{\left(r_{\lambda+\Delta\lambda}^{*}(\bx,\mset)\right)}\right)\right]-\mathbb{E}_{(\bx,\mset)}\left[c\left(h^{\left(r_{\lambda}^{*}(\bx,\mset)\right)}\right)\right]\right|\\
 & \stackrel{(a)}{\le}\mathbb{E}_{(\bx,\mset)}\left|c\left(h^{\left(r_{\lambda+\Delta\lambda}^{*}(\bx,\mset)\right)}\right)-c\left(h^{\left(r_{\lambda}^{*}(\bx,\mset)\right)}\right)\right|\\
 & =\mathbb{E}_{\mset}\int_{x\in\mathscr{X}}\left|c\left(h^{\left(r_{\lambda+\Delta\lambda}^{*}(\bx,\mset)\right)}\right)-c\left(h^{\left(r_{\lambda}^{*}(\bx,\mset)\right)}\right)\right|p(\bx)\,\mathrm{d}\bx\defeq\spadesuit,
\end{align*}
where we applied Jensen's inequality at $(a)$. Define $\mathcal{S}_{\Delta\lambda}$
to be the set of input points for which $r_{\lambda}^{*}$ and $r_{\lambda+\Delta\lambda}^{*}$
make different decisions. Precisely,
\begin{align*}
\mathcal{S}_{\Delta\lambda}(\mset) & \defeq\left\{ \bx\mid r_{\lambda+\Delta\lambda}^{*}(\bx,\mset)\neq r_{\lambda}^{*}(\bx,\mset)\right\} .
\end{align*}
Let $c_{\mathrm{max}}\defeq\max_{h\in\mset_{\mathrm{all}}}c(h)$ i.e.,
the maximum per-query cost of any LLM in our finite universe. The
quantity $\spadesuit$ can be bounded further with
\begin{align*}
\spadesuit & = \mathbb{E}_{\mset}\int_{x\in\mathcal{S}_{\Delta\lambda}(\mset)}\left|c\left(h^{\left(r_{\lambda+\Delta\lambda}^{*}(\bx,\mset)\right)}\right)-c\left(h^{\left(r_{\lambda}^{*}(\bx,\mset)\right)}\right)\right|p(\bx)\,\mathrm{d}\bx\\
 & \le 2 c_{\mathrm{max}}\mathbb{E}_{\mset}\int_{x\in\mathcal{S}_{\Delta\lambda}(\mset)}p(\bx)\,\mathrm{d}\bx\\
 & = 2c_{\mathrm{max}}\mathbb{E}_{\mset}\mathbb{P}(\mathcal{S}_{\Delta\lambda}(\mset)).
\end{align*}
The proof now amounts to showing that $\lim_{\Delta\lambda\to0}\mathbb{E}_{\mset}\mathbb{P}(\mathcal{S}_{\Delta\lambda}(\mset))=0$.

Define a shorthand $\lambda_{m,k}(\bx,\mset)\defeq\frac{\text{\ensuremath{\mathbb{E}_{\by|\bx}\left[\ell(\bx,\by,h^{(k)})\right]}}-\mathbb{E}_{\by|\bx}\left[\ell(\bx,\by,h^{(m)})\right]}{c(h^{(m)})-c(h^{(k)})}$
where $h^{(m)},h^{(k)}\in\mset$. We start by expanding $\mathcal{S}_{\Delta\lambda}(\mset)$
to have
\begin{align*}
\mathcal{S}_{\Delta\lambda}(\mset) & =\left\{ \bx\mid r_{\lambda+\Delta\lambda}^{*}(\bx,\mset)\neq r_{\lambda}^{*}(\bx,\mset)\right\} \\
 & \stackrel{(a)}{\subseteq}\left\{ \bx\mid\exists k\neq m:\thinspace\lambda_{m,k}(\bx,\mset)\in[\lambda,\lambda+\Delta\lambda]\right\} \\
 & =\bigcup_{m\neq k}\left\{ \bx\mid\lambda_{m,k}(\bx,\mset)\in[\lambda,\lambda+\Delta\lambda]\right\} \\
 & \defeq\mathcal{F}_{\Delta\lambda}(\mset),
\end{align*}
where at $(a)$ the subset relationship is due to Lemma \ref{lem:routed_model_changes}.
Since $\mathcal{S}_{\Delta\lambda}(\mset)\subseteq\mathcal{F}_{\Delta\lambda}(\mset)$,
we have
\begin{align*}
\spadesuit & \le 2 c_{\mathrm{max}}\mathbb{E}_{\mset}\mathbb{P}(\mathcal{F}_{\Delta\lambda}(\mset))\\
 & \stackrel{(a)}{\le} 2 c_{\mathrm{max}}\mathbb{E}_{\mset}\sum_{m\neq k}\mathbb{P}\left(\left\{ \bx\mid\lambda_{m,k}(\bx,\mset)\in[\lambda,\lambda+\Delta\lambda]\right\} \right)\\
 & = 2 c_{\mathrm{max}}\mathbb{E}_{\mset}\sum_{m\neq k}\mathbb{P}\left(\lambda_{m,k}(\bx,\mset)\in[\lambda,\lambda+\Delta\lambda]\right),
\end{align*}
where at $(a)$ we use the union bound. Note that there are a finite number
of summands because $\mset_{\mathrm{all}}$ is finite. Since $\bx$
is a continuous random variable, for any $\mset$, $\lambda_{m,k}(\bx,\mset)$
is also a continuous random variable. As $\Delta\lambda\to0$, we
have $\mathbb{P}\left(\lambda_{m,k}(\bx,\mset)\in[\lambda,\lambda+\Delta\lambda]\right)\to0$
for any $m,k$. This means $\spadesuit\to0$ and thus $\lim_{\Delta\lambda\to0}|C(r_{\lambda+\Delta\lambda}^{*})-C(r_{\lambda}^{*})|=0$.
\end{proof}

\begin{prop}
\label{prop:uncon_exists}
Suppose $\mathbb{P}(\by|\bx)$ and $\mathbb{P}(\bx)$ are continuous
in $\bx$. Let $r_{0}^{*}\in\arg\min_{r}L(r,0)$ where $L(r,\lambda)=S(r)+\lambda \cdot C(r)$ (see also \eqref{eq:shortuncon}). For any budget $B\in(0,C(r_{0}^{*})]$
in the constrained problem in \eqref{eq:risk} (or equivalently in \eqref{eq:shortcon}), there exists $\lambda_{\mathfrak{H}}\ge0$
such that the minimiser of $L(r,\lambda_{\mathfrak{H}})$ 
also minimises (\ref{eq:risk}).
\end{prop}

\begin{proof}
Since $\mathbb{P}(\by|\bx)$ and $\mathbb{P}(\bx)$ are continuous
in $\bx$, and $\lambda\mapsto C(r_{\lambda}^{*})$ is continuous
(by Lemma \ref{lem:cost_continuous}), there exists $\lambda_{B}\ge0$
such that $C(r_{\lambda_{B}}^{*})=B$. Applying Lemma \ref{lem:lagrange_minimiser}
with this choice of $\lambda_{B}$ implies that
\begin{align*}
S(r_{\lambda_{B}}^{*}) & \le S(r),
\end{align*}
for any $r$ such that $C(r)\le C(r_{\lambda_{B}}^{*})=B$. This proves
that $r_{\lambda_{B}}^{*}$ is the minimiser of (\ref{eq:risk}).
Setting $\lambda_{\mathfrak{H}}=\lambda_{B}$ completes the proof. 
\end{proof}

\subsection{Proof of Proposition \ref{prop:optimal_rule}}
\label{sec:proof_opt_rule}
\begin{prop*}[Restated]
Assume that $\bx \sim \mathbb{P}$ and $\eta(\bx) = \mathbb{P}(\by \mid \bx)$ are  continuous random variables. 
Then, for any LLM meta-distribution $\mathfrak{H}$,
and budget $B\in(0,C(r_{0}^{*})]$ (see \S\ref{sec:proof} for the definitions of $C$ and $r_0$), there exists $\lambda_{\mathfrak{H}} \ge 0$ such that
the optimal dynamic router $r^{*}$ for the constrained optimization
in \eqref{eq:risk} is
\begin{equation*}
r^{*}(\bx, \mathscr{H}) = \underset{m\in[ | \mathscr{H} | ]}{\argmin} \, \left[ \mathbb{E}_{\by\mid\bx}\left[\ell(\bx,\by,h^{(m)})\right]
+\lambda_{\mathfrak{H}} \cdot c(h^{(m)}) \right].
\end{equation*}
\end{prop*}

Note that the continuity assumption on the data distribution applies to the setting where $\bx$ can be represented via a fixed-length sentence embedding (i.e., a Euclidean vector). This is the primary setting of this work.

\begin{proof}
Under the continuity assumption on $\mathbb{P}$, by Proposition~\ref{prop:uncon_exists}, there exists $\lambda_{\mathfrak{H}}) \ge 0$ such that the minimiser of the unconstrained objective $L(r, \lambda_{\mathfrak{H}})$ (see $\eqref{eq:shortuncon}$) also minimises $\eqref{eq:risk}$.
By Lemma~$\ref{lem:opt_uncon}$, $r^*$ is established.

\end{proof}

\subsection{Proof of Proposition \ref{prop:cluster-regret-bound}}
\label{app:cluster-regret-bound}

\begin{prop*}[Restated]
For a set of LLMs $\mathscr{H}$, 
let $r^*$ denote the Bayes-optimal routing rule in Proposition \ref{prop:optimal_rule}. 
For any $\bx \in \XCal$ and $h^{(m)} \in \mathscr{H}$, let:
\begin{align*}
    \Delta_{k}(\bx, h^{(m)}) = \left|\mathbb{P}_{\by|\bx,z=k}\left[h(\bx)\neq\by\right] \,-\, \LLMEmbedScalar^*_{k}(h^{(m)}) \right|.
\end{align*}
Let $R_{01}(r, \msetNew) \defEq \sum_{n} \mathbb{P} \left[ \hNew^{( n )}(\bx)\neq\by \land r(\bx, \msetNew)=m \right]$  denote the 0-1 risk.
Then under the regularity condition on $\mathbb{P}$ in Proposition \ref{prop:optimal_rule}, the difference in 0-1 risk between $\tilde{r}^*$ and $r^*$ can be bounded as:
\begin{align*}
{\mathbb{E}_{\msetNew}\left[R_{01}(\tilde{r}^*,\msetNew)\right] ~\leq~ \mathbb{E}_{\msetNew}\left[{R}_{01}(r^*,\msetNew)\right]}  
 \,+\,
\mathbb{E}_{(\bx\msetNew)}\left[\max_{m\in[|\msetNew|], k\in[K]}\,\Delta_{k}(\bx, \hNew^{( n )})\right].
\end{align*}
\end{prop*}

For simplicity, we will refer to $\msetNew$ simply by $\mathscr{H}$, and $\hNew^{( m)}$ by $h^{(m)}$.  We define a proxy risk objective:
\begin{align}
\tilde{R}_{01}(r,\mathscr{H})
& =
\mathbb{E}_{\bx,z}\left[
\sum_{m \in [|\mathscr{H}|]} \LLMEmbedScalar^*_{z}(h^{(m)})\cdot\1\left[r(\bx, \mset)=m\right]\right],
\label{eq:proxy-risk}
\end{align}
where the expectation is over the joint distribution over $(\bx,z)$ (and not $(\bx, \by)$). 

Consider solving a variant of the constrained optimization problem in \eqref{eq:risk} where the original risk objective is replaced with the proxy risk in \eqref{eq:proxy-risk}:
\begin{align}
 & \min_{r}\thinspace\mathbb{E}_{\mathscr{H}}\left[ \tilde{R}_{01}(r,\mathscr{H}) \right]\nonumber \\
 & \text{s.t.}~~~\thinspace\mathbb{E}_{(\bx, \by, \mathscr{H})}\left[\sum_{m\in[|\mathscr{H}|]}c^{(m)} \cdot \1[r(\bx, \mset)=m]\right]\le B.\label{eq:proxy-optimization}
 \end{align}

We can then show that the optimal solution to
above proxy constrained optimization problem admits the same form as the cluster-based routing rule $\tilde{r}^*$ in \eqref{eq:cluster-rule-population}:
\begin{lemma}
\label{lem:cluster-opt-rule}
Under the assumption on $\mathbb{P}$ in Proposition \ref{prop:cluster-regret-bound}, for any set of models $\mathscr{H}$, the minimizer of the proxy constrained optimization problem in \eqref{eq:proxy-optimization} is given by:
\begin{align*}
\tilde{r}^{*}(\bx, \mset) 
 & =\underset{m\in[|\mathscr{H}|]}{\argmin}\sum_{k\in[K]}\mathbb{P}(z=k|\bx) \cdot \LLMEmbedScalar^*_{k}(h^{(m)}) +\lambda \cdot c^{(m)},
\end{align*}
for some $\lambda \geq 0$.
\end{lemma}

We will also find it useful to bound the difference between the original risk $R_{01}(r, \mset)$ and the proxy risk in \eqref{eq:proxy-risk}:
\begin{lemma}
\label{lem:cluster-Delta-bound}
For any routing rule $r$ and fixed $\mset$, 
\[
\left|R_{01}(r,\mathscr{H}) - \tilde{R}_{01}(r,\mathscr{H})\right| \,\leq\,
\mathbb{E}_{\bx}\left[\max_{m\in[|\mathscr{H}|], k\in[K]}\,\Delta_{k}(\bx, h^{(m)})\right].
\]
\end{lemma}

We are now ready to prove Proposition \ref{prop:cluster-regret-bound}:
\begin{proof}
The excess risk we wish to bound is given by:
\begin{align*}
    \lefteqn{
    \mathbb{E}_{\mathscr{H}}\left[ 
    R_{01}(\tilde{r}^*,\mathscr{H}) \,-\, {R}_{01}(r^*,\mathscr{H})
    \right]
    }
    \\ 
    &=  
    \mathbb{E}_{\mathscr{H}}\left[ 
    \Big\{R_{01}(\tilde{r}^*,\mathscr{H})
    \,-\, \tilde{R}_{01}(\tilde{r}^*,\mathscr{H})\Big\}
    \,+\, \tilde{R}_{01}(\tilde{r}^*,\mathscr{H})
    \,-\, \tilde{R}_{01}({r}^*,\mathscr{H})
    \,+\, \Big\{\tilde{R}_{01}({r}^*,\mathscr{H})
    \,-\, {R}_{01}(r^*,\mathscr{H})\Big\} 
    \right]
    \\
    &\stackrel{(a)}{\leq}
    \mathbb{E}_{\mathscr{H}}\left[ \tilde{R}_{01}(\tilde{r}^*,\mathscr{H}) \right]
    \,-\, 
    \mathbb{E}_{\mathscr{H}}\left[ \tilde{R}_{01}({r}^*,\mathscr{H}) \right] + 2\cdot\mathbb{E}_{(\bx, \mathscr{H})}\left[\max_{m\in[|\mathscr{H}|], k\in[K]}\,\Delta_{k}(\bx, h^{(m)})\right]\\
    &\stackrel{(b)}{\leq}
    0 + 2\cdot\mathbb{E}_{(\bx, \mathscr{H})}\left[\max_{m\in[|\mathscr{H}|], k\in[K]}\,\Delta_{k}(\bx, h^{(m)})\right],
\end{align*}
as desired. 
To derive (a), we apply Lemma \ref{lem:cluster-Delta-bound} to bound the first and third terms. To derive (b), we use the fact that $\tilde{r}^*$  is the minimizer of the proxy-risk $ \mathbb{E}_{\mathscr{H}}\left[ \tilde{R}_{01}(\cdot,\mathscr{H}) \right] $ subject to the budget constraint in  \eqref{eq:proxy-optimization}; since $r^*$ also  satisfies the same budget constraint, it has an equal or higher expected risk than $\tilde{r}^*$.
\end{proof}

We now prove Lemmas \ref{lem:cluster-opt-rule} and \ref{lem:cluster-Delta-bound}.

\begin{proof}[Proof of Lemma \ref{lem:cluster-opt-rule}]
Under the assumption on $\mathbb{P}$, the constrained problem in \eqref{eq:proxy-optimization}
is equivalent to minimizing the following Lagrangian objective for
some Lagrange multiplier $\lambda$ \citep{neyman1933ix}:
\begin{align*}
\mathcal{L} & =\mathbb{E}_{(\bx,z,\mset)}\left[\sum_{m \in [|\mathscr{H}|]}\LLMEmbedScalar^*_{z}(h^{(m)}) \cdot\1\left[r(\bx, \mset)=m\right]\right]+\lambda\cdot\mathbb{E}_{(\bx,\by,\mathscr{H})}\left[\sum_{m\in[|\mathscr{H}|]}\1(r(\bx,\mset)=m)\cdot c^{(m)}\right]\\
 & \stackrel{(a)}{=}\mathbb{E}_{\mset}\mathbb{E}_{\bx}\mathbb{E}_{z|\bx}\left[\sum_{m\in[|\mathscr{H}|]}\1(r(\bx, \mset)=m)\cdot\left\{ \LLMEmbedScalar^*_{z}(h^{(m)}) +\lambda\cdot c^{(m)}\right\} %
 \right]\\
 & =\mathbb{E}_{\mset}\mathbb{E}_{\bx}\bigg[
 \underbrace{
 \sum_{m\in[|\mathscr{H}|]}\1(r(\bx,\mset)=m)\cdot\Big\{ 
 \sum_{k\in[K]}\mathbb{P}(z=k|\bx) \cdot  \LLMEmbedScalar^*_{k}(h^{(m)}) 
 +\lambda\cdot c^{(m)}\Big\}
 }_{\mathcal{L}_{\mset,\bx}}
 \bigg],
\end{align*}
where (a) uses the fact that the draw of $\mset$ is independent of the draw of $(\bx,\by)$. 
The last line makes it clear that for any fixed $\mset$ %
and any fixed $\bx$, to minimize the overall loss, the router ought to
route to the model that has the lowest cost-adjusted loss $\mathcal{L}_{\mset,\bx}$. Thus, 
\begin{align*}
\tilde{r}^{*}(\bx,\mset) 
 & =\underset{m\in[|\mathscr{H}|]}{\argmin}\sum_{k\in[K]}\mathbb{P}(z=k|\bx) \cdot \LLMEmbedScalar^*_{k}(h^{(m)}) +\lambda \cdot c^{(m)}.
\end{align*}
\end{proof}

\begin{proof}[Proof of Lemma \ref{lem:cluster-Delta-bound}]
Expanding the original risk, we have:
\begin{align}
R_{01}(r,\mathscr{H})
& =\mathbb{E}_{(\bx, \by)}\left[\sum_{m\in[|\mathscr{H}|]}\1\left[h^{(m)}(\bx)\neq\by\right]\cdot\1\left[r(\bx, \mset)=m\right]\right]
\nonumber
\\
& = \mathbb{E}_{\bx}\bigg[\sum_{m\in[|\mathscr{H}|]}\mathbb{E}_{\by|\bx}\left[\1\left[h^{(m)}(\bx)\neq\by\right]\cdot\1\left[r(\bx, \mset)=m\right]\right]\bigg]
\nonumber
\\
& = \mathbb{E}_{\bx}\left[\sum_{m\in[|\mathscr{H}|]}\mathbb{P}_{\by|\bx}\left[h^{(m)}(\bx)\neq\by\right]\cdot\1\left[r(\bx, \mset)=m\right]\right]
\nonumber
\\
& = \mathbb{E}_{\bx}\left[\sum_{m\in[|\mathscr{H}|]}\sum_{k \in [K]}\pi_k \cdot \mathbb{P}_{\by|\bx, z=k}\left[h^{(m)}(\bx)\neq\by\right]\cdot\1\left[r(\bx, \mset)=m\right]\right].
\label{eq:R01-expand}
\end{align}

Recall that:
\begin{align*}
\Delta_{k}(\bx, h^{(m)}) &= 
\left|\mathbb{P}_{\by\,|\,\bx, z=k} \left[h^{(m)}(\bx)\neq\by \right] \,-\, \LLMEmbedScalar^*_{k}(h^{(m)})\right|
\\
&=\left|\mathbb{P}_{\by\,|\,\bx, z=k} \left[h^{(m)}(\bx)\neq\by \right] \,-\, \mathbb{P}_{\bx', \by'\,|\,z=k}\left[h^{(m)}(\bx')\neq\by'\right] \right|.
\end{align*}

We may next bound \eqref{eq:R01-expand} in terms of $\Delta_{m}(\bx, h^{(m)})$:
\begin{align*}
\lefteqn{R_{01}(r,\mathscr{H})} \\
&\leq
\mathbb{E}_{\bx}\left[\sum_{m\in[|\mathscr{H}|]}\sum_{k \in [K]}\pi_k \cdot \Big( 
\mathbb{P}_{\bx',\by'\,|\,z=k}\left[
    h^{(m)}(\bx')\neq\by'
\right]
+ \Delta_{k}(\bx, h^{(m)})
\Big)
\cdot\1\left[r(\bx, \mset)=m\right]\right]
\\
&=
\mathbb{E}_{\bx}\left[\sum_{k \in [K]}\pi_k \cdot\sum_{m\in[|\mathscr{H}|]} \mathbb{P}_{\bx', \by'|z=k}\left[h^{(m)}(\bx')\neq\by'\right]\cdot\1\left[r(\bx, \mset)=m\right]\right] \\
&\hspace{5cm}
\,+\,
\mathbb{E}_{\bx}\left[\sum_{k \in [K]}\pi_k \cdot \sum_{m\in[|\mathscr{H}|]}\Delta_{k}(\bx, h^{(m)}) \cdot\1\left[r(\bx, \mset)=m\right]\right]
\\
& \stackrel{(a)}{\leq}
\mathbb{E}_{\bx}\left[\sum_{k \in [K]}\pi_k \cdot \sum_{m\in[|\mathscr{H}|]} \mathbb{P}_{\bx', \by'|z=k}\left[h^{(m)}(\bx')\neq\by'\right]\cdot\1\left[r(\bx, \mset)=m\right]\right] 
\,+\,
\mathbb{E}_{\bx}\left[\sum_{k \in [K]}\pi_k \cdot \max_{m\in[|\mathscr{H}|]}\,\Delta_{k}(\bx, h^{(m)})\right] \\
& \stackrel{(b)}{\leq}
\mathbb{E}_{\bx}\left[\sum_{k \in [K]}\pi_k \cdot \sum_{m\in[|\mathscr{H}|]} \mathbb{P}_{\bx', \by'|z=k}\left[h^{(m)}(\bx')\neq\by'\right]\cdot\1\left[r(\bx, \mset)=m\right]\right] 
\,+\,
\mathbb{E}_{\bx}\left[\max_{m\in[|\mathscr{H}|], k\in[K]}\,\Delta_{k}(\bx, h^{(m)})\right]\\
&=
\mathbb{E}_{(\bx, z)}\left[ \sum_{m\in[|\mathscr{H}|]}\LLMEmbedScalar^*_{z}(h^{(m)})\cdot\1\left[r(\bx, \mset)=m\right]\right] 
\,+\,
\mathbb{E}_{\bx}\left[\max_{m\in[|\mathscr{H}|], k\in[K]}\,\Delta_{k}(\bx, h^{(m)})\right]\\
&= \tilde{R}_{01}(r, \mathscr{H}) \,+\,
\mathbb{E}_{\bx}\left[\max_{m\in[|\mathscr{H}|], k\in[K]}\,\Delta_{k}(\bx, h^{(m)})\right].
\end{align*}
where $(a)$ uses the fact that $\sum_m \1\left[r(\bx, \mset)=m\right] = 1$ and $(b)$ follows from the fact that $\sum_k \pi_k = 1$.

One can similarly show that:
\[
R_{01}(r,\mathscr{H}) \,\geq\, \tilde{R}_{01}(r,\mathscr{H}) \,-\,
\mathbb{E}_{\bx}\left[\max_{m\in[|\mathscr{H}|], k\in[K]}\,\Delta_{k}(\bx, h^{(m)})\right],
\]
which completes the proof.
\end{proof}

\section{Zero Routing}
\label{app:pareto-random}
An elementary approach to multi-model routing is to  identify the %
points 
on the non-decreasing convex hull
of the set of cost-risk pairs $\{(c^{(m)}, R( h^{(m)}): m \in [M]\}$, and to route amongst the corresponding LLMs \citep{HuBieLi2024}. %
Specifically, given a budget $B \in ( c^{(1)}, c^{(M)} )$, we may pick the two nearest costs $c^{(m_1)} \leq B <  c^{(m_2)}$ from the non-decreasing convex hull, and route a query to LLM $h^{(m_1)}$ with probability $\frac{c^{(m_2)} - B}{c^{(m_2)} - c^{(m_1)}}$ and to $h^{(m_2)}$ with probability $\frac{B - c^{(m_1)}}{c^{(m_2)} - c^{(m_1)}}$. 

This approach can be seen as a \textit{random router} that randomizes between two LLMs, where the LLMs and mixing coefficients are chosen to maximize the expected quality on the validation sample, while satisfying the budget constraint.
Despite its simplicity (i.e., the routing decision being agnostic to the input), this approach was noted as a strong baseline in \citet{HuBieLi2024}.

\subsection{Special Case: Optimal Cluster Routing Rule for $K=1$}
When the number of clusters $K=1$, the routing rule in \eqref{eqn:plugin-dynamic-routing} returns the same LLM for all queries $\bx$, and is given by:
\begin{align}
\routerEstimator(\bx, \msetNew) & =
\underset{n \in [ \numSetNew ]}{\argmin} \, \big[ {\LLMEmbedScalar}( \hNew^{(n)} ) + \lambda\cdot c( \hNew^{(n)} ) \big],
\label{eq:cluster-routing-k-equals-1}
\end{align}
where ${\LLMEmbedScalar}( \hNew^{( n )} ) \defeq\frac{1}{N_{\rm val}}\sum_{(\bx, \by) \in S_{\rm val}} \1[\hNew^{( n )}(\bx)\ne\by].$  
This rule is closely aligned with the Pareto-random router. %

\begin{prop}
{For any $\lambda \in \mathbb{R}_{\geq 0}$, the routing rule in \eqref{eq:cluster-routing-k-equals-1} returns an LLM on the  non-decreasing \emph{convex hull} of the set of cost-risk pairs 
$\{(c( \hNew^{( n )} ), \routerEstimator_{01}( \hNew^{( n )} ): n \in [ \numSetNew ]\}$, 
where 
$\routerEstimator_{01}( \hNew^{( n )} ) = \frac{1}{N_{\rm val}}\sum_{(\bx, \by) \in S_{\rm val}} \1[\hNew^{( n )}(\bx)\ne\by]$.}
\label{prop:cluster-routing-k-equals-1}
\end{prop}
\begin{proof}

Suppose there exists a $\lambda_1 \in \mathbb{R}_{\geq 0}$ and $m_1 \in \argmin_m \, {\LLMEmbedScalar}(\hNew^{( m)})+\lambda_1\cdot c^{(m)}$ such that $(c^{(m_1)}, \routerEstimator_{01}( h^{(m_1)}))$ is not on the non-decreasing convex hull. Then there exists $h^{(m_2)} \in \mathscr{H}$ such that either $c^{(m_2)} < c^{(m_1)}$ and $\routerEstimator_{01}( h^{(m_2)}) \leq \routerEstimator_{01}( h^{(m_1)})$, or $c^{(m_2)} \leq c^{(m_1)}$ and $\routerEstimator_{01}( h^{(m_2)}) < \routerEstimator_{01}( h^{(m_1)})$. In either case, ${\LLMEmbedScalar}(\hNew^{( m_2)})+\lambda_1\cdot c^{(m_2)} = \routerEstimator_{01}( \hNew^{(m_2)}) +\lambda_1\cdot c^{(m_2)} < \routerEstimator_{01}( \hNew^{(m_1)}) +\lambda_1\cdot c^{(m_1)} = {\LLMEmbedScalar}(\hNew^{( m_1)})+\lambda_1\cdot c^{(m_1)}$, which contradicts the fact that $m_1 \in \argmin_m\, [ {\LLMEmbedScalar}(\hNew^{( m)})+\lambda_1\cdot c^{(m)} ]$. 
\end{proof}

\section{Experimental Setup}

We provide more details on the experiments discussed in \cref{sec:experiments}.

\subsection{Splitting Data and LLMs}
In the experiment on each of the four datasets (EmbedLLM \citep{ZhuWuWen2024}, MixInstruct \citep{Jiang:2023},  RouterBench \citep{HuBieLi2024}), and Math+Code dataset  \citet{Dekoninck:2025}), we split the data into three disjoint portions:  train, validation, and test. 
The set of all LLMs available in each dataset is also split into two disjoint sets: training models and testing models. The relationship of data splits and model splits is summarized in the following table.

\newcommand{\cmark}{\ding{51}}%
\newcommand{\xmark}{\ding{55}}%
\newcommand{\shadecell}{ \cellcolor{gray!25}}
\begin{center}
\begin{tabular}{l|>{\centering\arraybackslash}p{49mm}|>{\centering\arraybackslash}p{14mm}|>{\centering\arraybackslash}p{28mm}}
\toprule
  \diaghead(-2,1){justpadspaceeee}{Model split}{Data split}  & Train & Validation & Test  \\
\midrule 
  Training models & \shadecell \cmark & \shadecell \cmark & \xmark \\
  Testing models & \xmark    & \shadecell \cmark & \shadecell \cmark \\ 
\bottomrule
\end{tabular}
\end{center}

\begin{itemize}
    \item \textbf{Training set}. The training examples are meant for router training. Only information of the training models (not testing models) is available in this data portion.  The only exceptions are the clairvoyant oracle baselines which are allowed access to correctness labels of testing models  on training examples. In other words, unlike other baselines, these oracle methods observe all models during training, and are  trained on both training and validation portions. These baselines are meant to establish performance achievable if a router has access to all models.
    \item \textbf{Validation set}. The validation examples are meant to be used to represent new LLMs. For instance, for our proposed \ourMethod{} ($K$-Means) approach, the validation set is used to compute per-cluster performance metrics of each testing LLM observed at test time, to represent it as a feature vector.
    \item \textbf{Test set}. The test examples are only used for evaluating routing methods.
\end{itemize}

Testing models represent new models that arrive at deployment time, and are not available for training (except to the clairvoyant fixed-pool router baseline).
Training models are meant for router training. For instance, our  \ourMethod{} ($K$-Means) approach learns to route among the training models, and is tested on the test set to route among the testing models.

For the Math+Code dataset alone (in \Cref{fig:three_experiments}), we have no training LLMs; so the training sample is unlabeled.

\subsection{Evaluation: Deferral Curve}
\label{sec:deferral-curve}

Routing performance may be assessed via a \emph{deferral curve}~\citep{Jitkrittum:2023,WanAugRus2024,HuBieLi2024}
$\mathscr{C} = \{ ( B, R( h_{\rm RM}( \cdot, r_B ) ) \colon B \in [ c^{(1)}, c^{(M)} ] \}$,
tracing the tradeoff between the cost budget $B$ and loss of the resulting routed model. 
Specifically, one varies the cost budget $B \in [ c^{(1)}, c^{(M)} ]$;
computes a router $r_B( \cdot )$ for this budget;
and plots
the resulting expected loss  $R( h_{\rm RM}( \cdot, r_B ) )$.
We may also use a quality metric (e.g., accuracy) in place of the loss to 
capture quality-cost trade-offs.

\subsection{Implementation Details of \ourMethod{} (LearnedMap)}
\label{app:implementation-details}

In this section, we give details on the architecture and training
of our proposed \ourMethod{} (LearnedMap).

\paragraph{Architecture}

Recall from \Cref{sec:two_tower} that the LearnedMap attempts to learn $\bx\mapsto\boldsymbol{\Phi}_{\mathrm{clust}}(\bx;\boldsymbol{\theta})$
parameterised by $\boldsymbol{\theta}$. This is a function that maps
an input prompt $\bx$ to  a probability vector $\boldsymbol{\Phi}_{\mathrm{clust}}(\bx;\boldsymbol{\theta})\in\Delta^{K}$
where $\Delta^{K}$ denotes the $K$-dimensional probability simplex.
In experiments, $\boldsymbol{\Phi}_{\mathrm{clust}}(\cdot;\boldsymbol{\theta})$
is defined by an MLP with two hidden layers:
\begin{align}
\boldsymbol{\Phi}_{\mathrm{clust}}(\bx;\boldsymbol{\theta}) & \defeq \left(\mathrm{Softmax}\circ\mathrm{FC}(M)\circ\mathsf{H}'\circ\mathsf{H}\circ\mathrm{BN}\circ\boldsymbol{\varphi}\right)(\bx),\label{eq:learnedmap_arch}\\
\mathsf{H} & \defeq\left[\mathrm{ReLU}\circ\mathrm{BN}\circ\mathrm{FC}(128)\right],\label{eq:learnedmap_hidden}
\end{align}
where:
\begin{itemize}
\item $\circ$ denotes function composition, 
\item $\mathsf{H}$ and $\mathsf{H}'$ denote the two, separate (i.e., no
parameter sharing) hidden layers of the same architecture as described
in (\ref{eq:learnedmap_hidden}),
\item $\mathrm{FC}(z)$ denotes a fully connected layer with $z\in\mathbb{N}$
outputs,
\item $\mathrm{ReLU}$ denotes the rectified linear unit i.e., $\mathrm{ReLU}(a)\defeq\max(0,a),$
\item $\mathrm{Softmax}$ denotes a softmax layer, 
\item $\mathrm{BN}$ denotes the batch normalization layer,
and
\item $\boldsymbol{\varphi}$ denotes a frozen text embedding. 
\end{itemize}
\begin{center}
Recall from the experiments in \Cref{sec:experiments} that
we use Gecko \cite{LeeDaiRen2024} for the text embedding $\boldsymbol{\varphi}$.
\par\end{center}

\paragraph{Training}

We use Keras 
for implementing (\ref{eq:learnedmap_arch}).
In all experiments on EmbedLLM, MixInstruct, and RouterBench, the
Learned Cluster Map is trained for only 5 epochs using Adam 
as the optimization algorithm. We observe that training for too long
can lead to overfitting. Training batch size is set to 64, and the learning rate is 0.005. Since $\boldsymbol{\varphi}$
is frozen, training $\boldsymbol{\Phi}_{\mathrm{clust}}(\cdot;\boldsymbol{\theta})$
amounts to training the MLP with two hidden layers. It is sufficient
to use CPUs for training. For one trial, training takes only a few minutes to complete.

\newcommand{\imgw}{0.9\textwidth}
\begin{figure}[H]
  \begin{minipage}[b]{.98\linewidth}
  
\begin{center}
 EmbedLLM \hspace{27mm} MixInstruct \hspace{25mm} RouterBench
 \end{center}
 \vspace{2mm}
 
    \begin{subfigure}{0.32\textwidth}
        \includegraphics[width=\imgw]{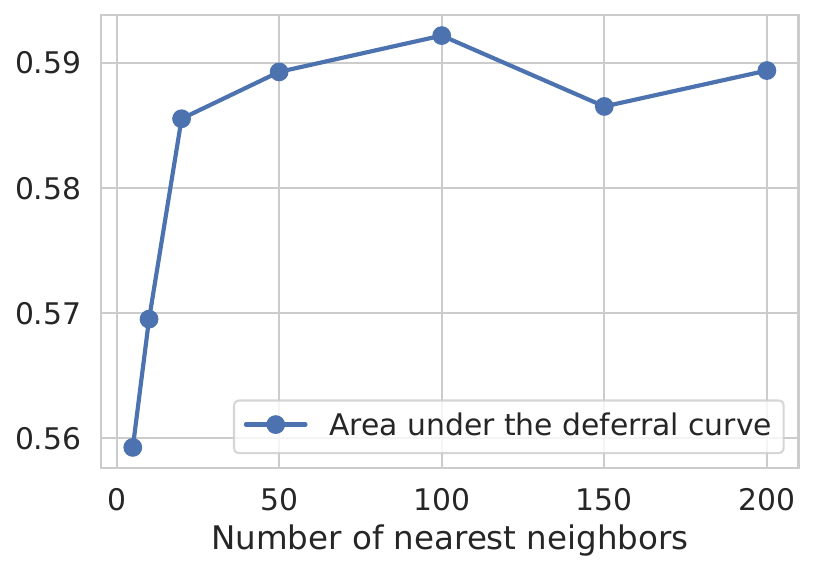}
        \caption{$K$-NN}
    \end{subfigure}
    \begin{subfigure}{0.32\textwidth}
        \includegraphics[width=\imgw]{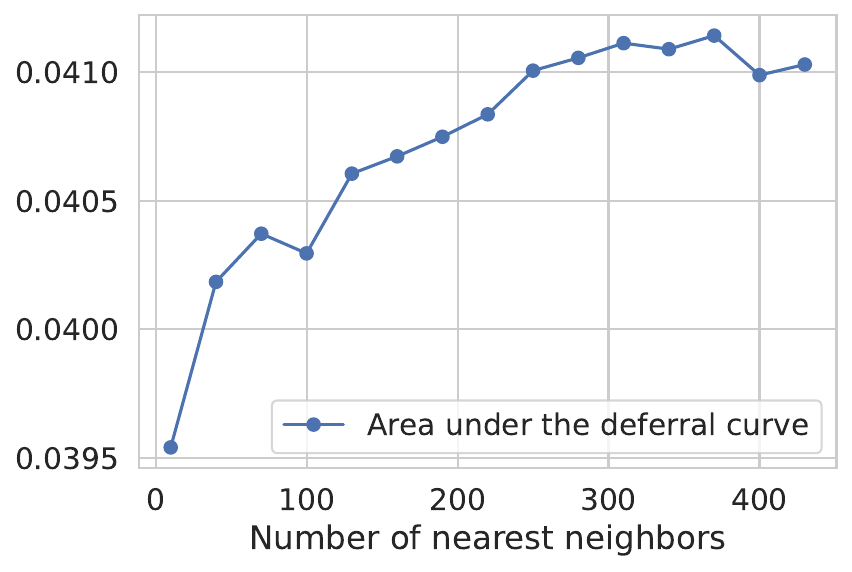}
        \caption{$K$-NN}
    \end{subfigure}
    \begin{subfigure}{0.32\textwidth}
        \includegraphics[width=\imgw]{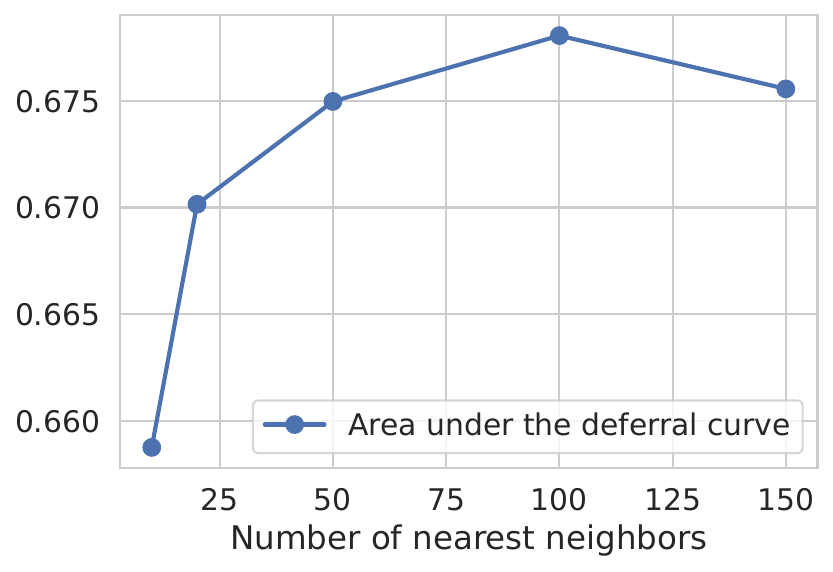}
        \caption{$K$-NN}
    \end{subfigure}
    
    \begin{subfigure}{0.32\textwidth}
        \includegraphics[width=\imgw]{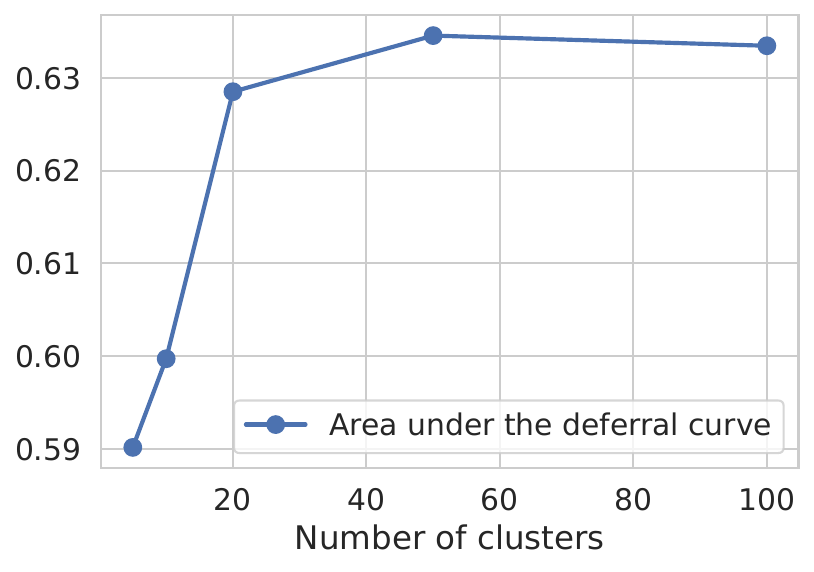}
        \caption{\ourMethod{} ($K$-Means)}
    \end{subfigure}
    \begin{subfigure}{0.32\textwidth}
        \includegraphics[width=\imgw]{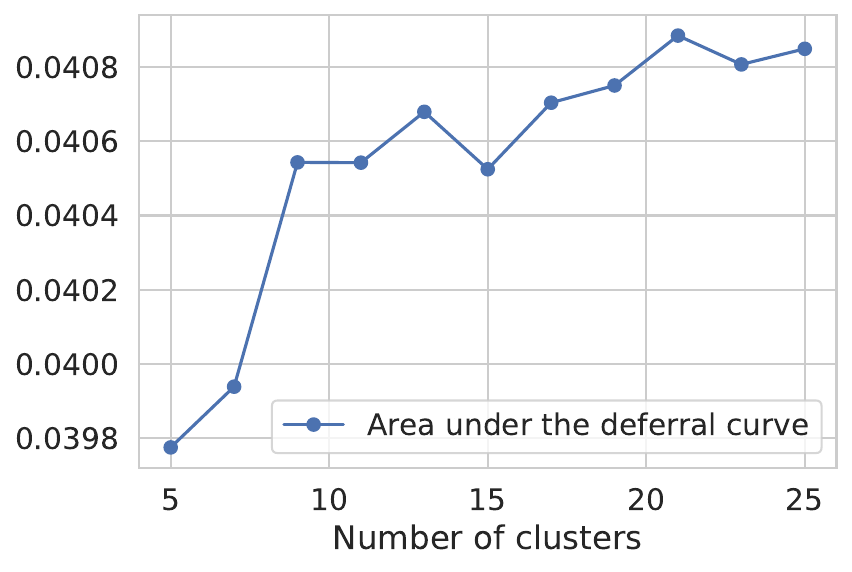}
        \caption{\ourMethod{} ($K$-Means)}
    \end{subfigure}
    \begin{subfigure}{0.32\textwidth}
        \includegraphics[width=\imgw]{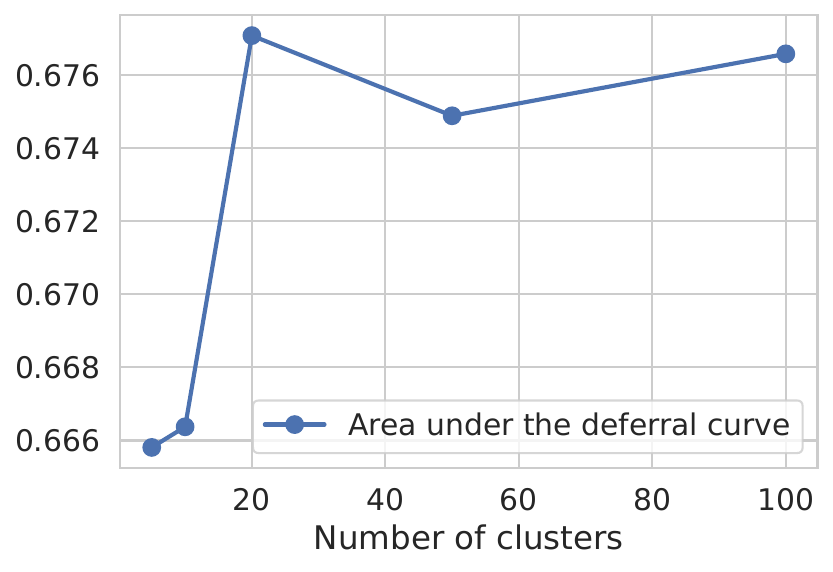}
        \caption{\ourMethod{} ($K$-Means)}
    \end{subfigure}
    
   \begin{subfigure}{0.32\textwidth}
        \includegraphics[width=\imgw]{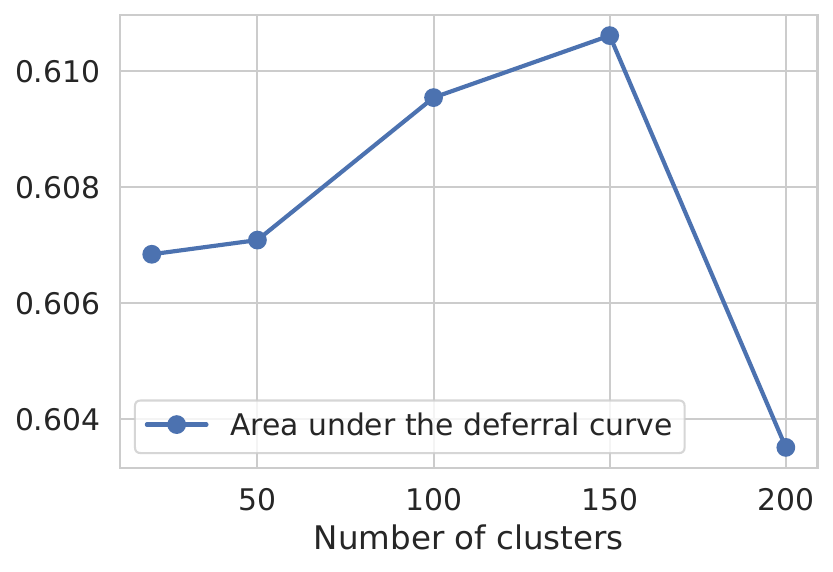}
        \caption{\ourMethod{} ($K$-Means Attributes}
    \end{subfigure}
    \begin{subfigure}{0.32\textwidth}
        \includegraphics[width=\imgw]{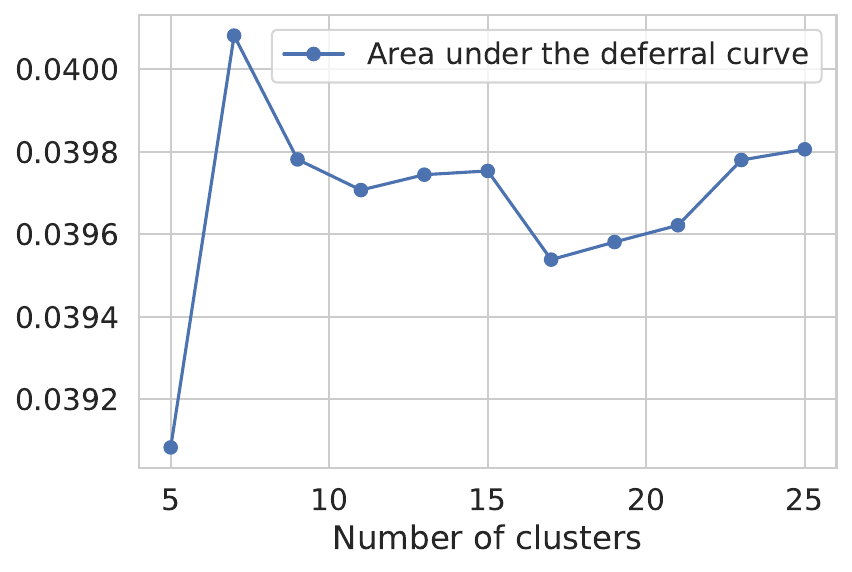}
        \caption{\ourMethod{} ($K$-Means Attributes}
    \end{subfigure}
    \begin{subfigure}{0.32\textwidth}
        \includegraphics[width=\imgw]{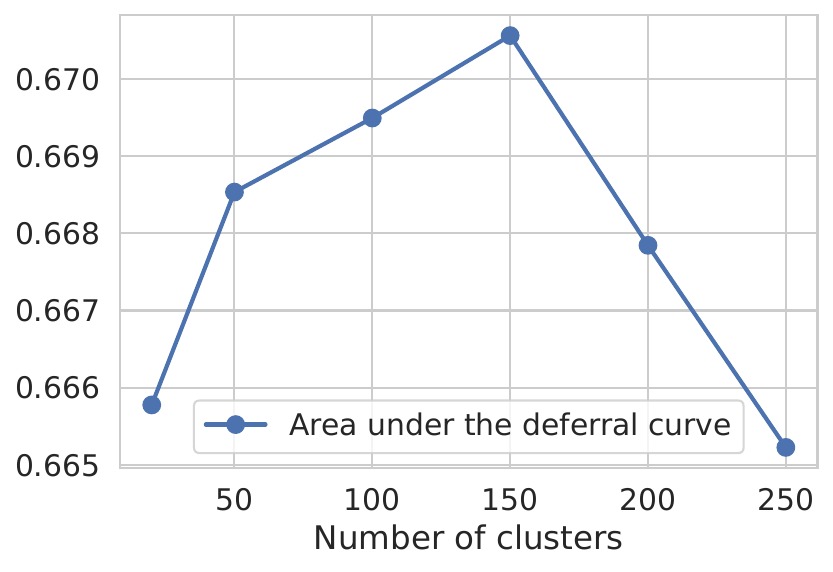}
        \caption{\ourMethod{} ($K$-Means Attributes}
    \end{subfigure}
    
    \begin{subfigure}{0.32\textwidth}
        \includegraphics[width=\imgw]{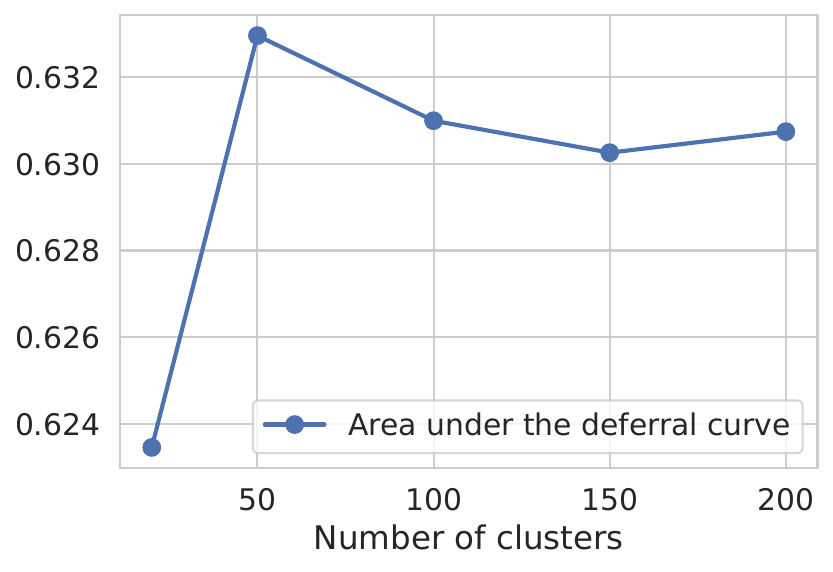}
        \caption{\ourMethod{} (LearnedMap)}
    \end{subfigure}
    \begin{subfigure}{0.32\textwidth}
        \includegraphics[width=\imgw]{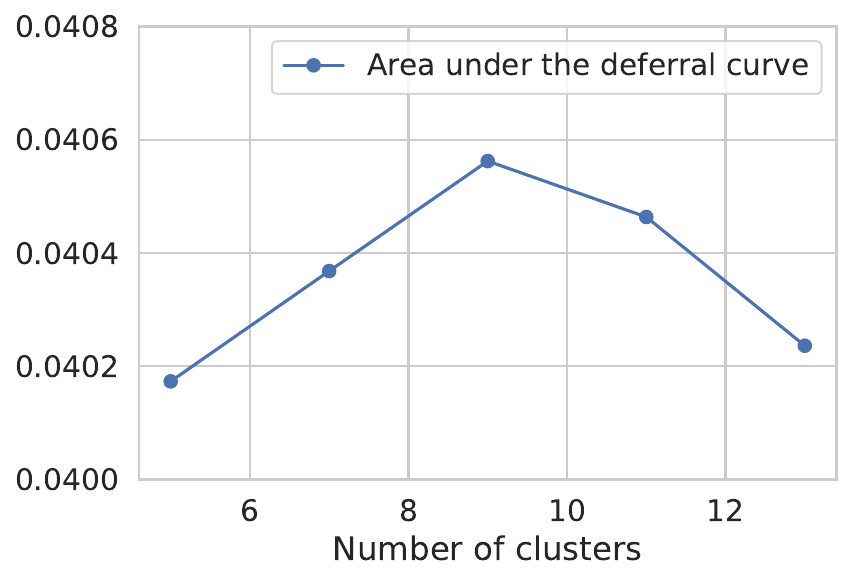}
        \caption{\ourMethod{} (LearnedMap)}
    \end{subfigure}
    \begin{subfigure}{0.32\textwidth}
        \includegraphics[width=\imgw]{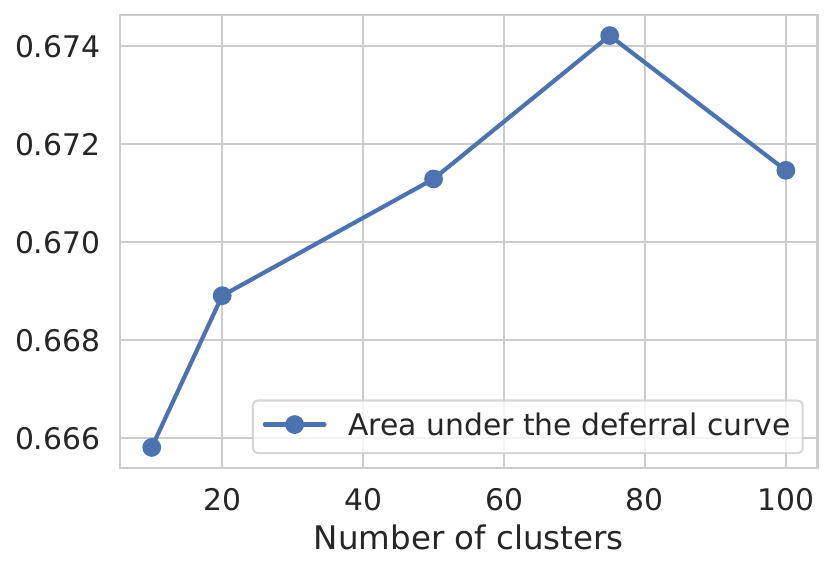}
        \caption{\ourMethod{} (LearnedMap)}
    \end{subfigure}
  \end{minipage}
  \caption{Validation performance of four methods considered in \cref{fig:three_experiments} and \cref{sec:chatbot_exp_details}: $K$-NN, \ourMethod{} ($K$-Means), \ourMethod{} ($K$-Means Attributes, and \ourMethod{} (LearnedMap). See \cref{sec:validate_k} for more details.
  }
  \label{fig:validate_k}
\end{figure}

\subsection{Router Cost}
All routing methods rely on a frozen prompt embedding model $\boldsymbol{\varphi}$. 
As described in \Cref{app:deferral}, $\boldsymbol{\varphi}$ is set be the token probability quantiles of the Gemma 2B model in the Headlines dataset. For other datasets, we use Gecko 1B model \citep{LeeDaiRen2024} for $\boldsymbol{\varphi}$. Importantly, all methods rely on the same embedding model, and thus share the same overhead for prompt embedding.

At inference time, for \ourMethod{} ($K$-Means), deciding the LLM to route to for a given test query involves computing the text embedding, and finding the nearest centroid out of $K$ centroids (recall that $K$-means is run only once on the training set). 
For \ourMethod{} (LearnedMap), after computing the prompt embedding, we pass the embedding through a small MLP as described in \eqref{eq:learnedmap_arch}.
Compared to the sizes of candidate LLMs (up to 60+B in EmbedLLM, for instance), invoking the text embedding model and performing a few operations after (centroid lookup, or invoking a small MLP) incur a negligible cost.

When plotting the deferral curves, we do not include the cost of the routing model. Including the router’s overhead will simply shift all the deferral curves to the right by a small amount, without changing the relative ordering of the methods.

\section{Additional Experimental Results}
\label{app:expts-additional}
We present additional experimental results we omitted in the main text.

\subsection{Hyper-parameter Choices \& Statistical Significance}
\label{sec:validate_k}
There are three methods that we consider in the experiments in \cref{sec:experiments} that depend on a hyperparameter $K$.
Specifically, $K$ in K-NN refers to the number of nearest neighbors, and $K$ in \ourMethod{} ($K$-means) and \ourMethod{} (LearnedMap) refers to the number of clusters. In \cref{fig:three_experiments}, we report the performance of these methods with the best $K$ found on each dataset separately.
We now describe the validation procedure we used to select the best $K$.

\paragraph{K-NN} 
For each candidate $K$, and each prompt in the validation set, find the $K$ nearest queries in the training set (in the Gecko embedding space). Route each prompt in the validation set to the most appropriate \emph{training} LLM according to the routing rule \eqref{eq:opt_rule01} where $\gamma^{(m)}$ is estimated with \eqref{eqn:knn-router}. Produce a deferral curve on the validation set, and compute the normalized area under such curve. Select $K$ that maximizes the area.
The list of candidate $K$'s is set to be from 5 to one third of the validation sample size.

\paragraph{\ourMethod{} ($K$-Means) }
For each candidate $K$, perform $K$-means on the training set using Gecko embeddings \citep{LeeDaiRen2024}.
Compute the feature vector representation of each \emph{training} LLM on the training set using \eqref{eq:cluster_accs}. 
For each prompt in the validation set, find the nearest cluster, and route the prompt to the most appropriate \emph{training} LLM according to the routing rule \eqref{eqn:plugin-dynamic-routing}. 
Produce a deferral curve on the validation set, and compute the normalized area under such curve. Select $K$ that maximizes the area.
The list of candidate $K$'s is from 3 to the number of validation sample size, divided by 50.

\paragraph{\ourMethod{} ($K$-Means Attributes)}
An alternate approach that we consider in Appendix \ref{sec:chatbot_exp_details} is to construct 
a binary vector of \emph{prompt attributes},
denoting whether a prompt possesses certain characteristics~\citep{Li:2024c,Li:2024d},
e.g.,
whether it 
requires multi-step reasoning,
seeks a single correct answer,
and so on.
These can be seen as a generalised 
``task''.
Compared to a general purpose text embedding,
such a representation is a coarser representation;
on the other hand,
for the purposes of model routing,
this can help mitigate overfitting.

Concretely, we parameterize the prompt embedding model to be  $\Phi(\bx) = \sigma(\mathbf{V}^\top \PretrainedQueryEmbed_{\text{Gecko}}(\bx))$ where $\mathbf{V} \in \mathbb{R}^{768 \times 7}$, and $\sigma$ denotes the sigmoid function.
Train each head $\mathbf{v}_j \in \mathbb{R}^{768}$  (with $\PretrainedQueryEmbed_{\text{Gecko}}$ frozen) by minimizing the sigmoid cross entropy to predict whether the $j$-th semantic attribute is active on each input prompt.  We use the seven prompt difficulty attributes as described in \citet{Li:2024c}, and prompt Gemini 1.5 Pro 002 to annotate each binary attribute on each training example.
Once the prompt embedding model $\Phi$ is trained, we freeze it, and perform the same hyperparameter selection procedure as used for K-means (Gecko) by replacing the Gecko embedding function with $\Phi$.

\paragraph{\ourMethod{} (LearnedMap) }
For each candidate $K$, perform $K$-means on the training set using Gecko embeddings.
Compute the feature vector representation of each \emph{training} LLM on the training set using \eqref{eq:cluster_accs}. 
Parameterize the cluster assignment map with a softmax-dense layer as described in  \cref{sec:two_tower}, resulting in a parameter $\boldsymbol{\theta}$ to learn. Learn the parameter by minimizing the binary sigmoid cross entropy on the training with the labels given by the correctness labels of the \emph{training} LLMs.
With the cluster map trained, for each prompt in the validation set, route the prompt to the most appropriate training LLM according to the routing rule~\eqref{eqn:plugin-dynamic-routing}.
Produce a deferral curve on the validation set, and compute the normalized area under such curve. Select $K$ that maximizes the area.
The list of candidate $K$'s is from 3 to the number of validation sample size, divided by 50.

The above tuning procedure for $K$ cannot be applied to the Math+Code dataset considered in \Cref{sec:experiments}, where we have no training LLMs. In this case, we set $K$ to $\sqrt{N_{\rm val}}$ for $K$-NN and to $ N_{\rm val} / 50 $ for $K$-Means.

\paragraph{Effect of $K$} \cref{fig:validate_k} shows the area under the deferral curve (on the validation set) vs candidate parameter $K$. Importantly, the testing models and the test set are never used in the above hyperparameter selection process.

\paragraph{MLP (Clairvoyant)} For this oracle baseline, as per \citep{HuBieLi2024}, we fix the number of hidden layers to two and the number of hidden nodes in each hidden layer to 100. 
We fit the MLP on the combined training and validation set to predict a quality score for each test LLM, and route via  \eqref{eqn:post-hoc}. For training, we once again ran  30 epochs od Adam with a learning rate of $10^{-3}$ and  batch size 64, and picked the checkpoint that yielded the best quality metric on the validation set.

\paragraph{Matrix Factorization (Clairvoyant)} For this oracle baseline, we fit a factorized model $\mathbf{A}\cdot \mathbf{B}$, with $\mathbf{A} \in \mathbb{R}^{\PretrainedQueryEmbedDim \times d}, \mathbf{B} \in  \mathbb{R}^{d \times m},$  where $\PretrainedQueryEmbedDim$ is the dimension of the pretrained query embedding, $d$ is an intermediate dimension, and $N$ is the number of test LLMs to route to. 
We choose $d$ using the same procedure used to pick the hyper-parameter $K$ above. For a query $\bx \in \XCal$ and pre-trained $\PretrainedQueryEmbed(\bx) \in \mathbb{R}^{\PretrainedQueryEmbedDim}$, this router predicts $N$ LLM scores via $\mathbf{A}\cdot \mathbf{B} \cdot \PretrainedQueryEmbed(\bx)$, and routes via  \eqref{eqn:post-hoc}.  For training, we ran 30 epochs of Adam with a learning rate of $10^{-3}$ and  batch size 64, and picked the checkpoint that yielded the best quality metric on the validation set.

\paragraph{Statistical Significance}  For the table in Figure \ref{fig:three_experiments} (Top), we repeat the experiments over 400 trials, and report  statistical significance (via the sign test) at signficance level $\alpha=0.01$. For the plots of area vs.\ validation sample size in Figure \ref{fig:three_experiments} (Bottom), we repeat the experiments over 100 trials for EmbedLLM, RouterBench, SPROUT o3-mini and 1000 trials for Math+Code, and report 96\% confidence intervals (i.e. two-sigma standard errors).

\subsection{Deferral Curves and Additional Comparisons}
\label{app:deferral}
As described in \Cref{sec:experiments}, for each of the 400 independent trials, we randomly split examples of each dataset into training, validation and testing. 
\Cref{fig:deferral-curves-app} presents the mean deferral curves of all the methods we consider.
All the statistics (Area (50\%), Area and QNC) reported in \Cref{fig:three_experiments} are derived from these curves. 
All results here are for the dynamic LLM pool setting as considered in \Cref{fig:three_experiments}.

\paragraph{EmbedLLM}
\Cref{fig:embedllm} presents deferral curves on the EmbedLLM problem. These results are the same as in \Cref{fig:deferral-curve-embedllm} where we add two oracle baselines: Matrix Factorization (MatFac) \citep{OngAlmWu2024,ZhuWuWen2024}, and MLP \citep{HuBieLi2024}.
We reemphasize that these two baselines are not designed for the dynamic pool setting, and are allowed to observe test LLMs during training to be applicable to our setting.

\paragraph{Headlines}
We consider the Headlines dataset as used in \citet{CheZahZou2023} consisting of 10K prompts in total. 
Each prompt asks an LLM to determine the price direction (up, down, neutral, or none) of an item in a list of news headlines.
There are 12 LLMs of various sizes in this dataset. 
Each LLM has a distinct cost (USD) of processing one prompt, which is a function of the input length, output length, and a fixed cost per request (see Table 1 in \citet{CheZahZou2023}).
We hold out six LLMs for training and the other six LLMs for testing:
\begin{itemize}
    \item \underline{Training LLMs}: 
        Textsynth FAIRSEQ, OpenAI GPT-4,  OpenAI GPT-Curie, Textsynth GPT-Neox, AI21 J1-Large, Cohere Xlarge;
    \item \underline{Test LLMs}:
        OpenAI ChatGPT, OpenAI GPT-3, Textsynth GPT-J, AI21 J1-Grande, AI21 J1-Jumbo, Cohere Medium.
\end{itemize}
In each of the 400 independent trials, we randomly partition the data into 4000, 400, 5600 examples for training, validation, and testing, respectively.
For Headlines, 
unlike other datasets, we do \emph{not} use Gecko for embedding  prompts.
Rather,
the frozen text embedding $\boldsymbol{\varphi}$ is constructed based on the output of the Gemma 2B model\footnote{Gemma 2B model: \url{https://huggingface.co/google/gemma-2b}.}:  for each prompt, its embedding is the seven equally spaced quantiles of the per-token probabilities of the output tokens from Gemma 2B.

The mean deferral curves over 400 trials are shown in \Cref{fig:headlines}. We observe that our proposed \ourMethod{} (LearnedMap) has higher accuracy than $K$-NN throughout the whole cost range.

\paragraph{MixInstruct}
MixInstruct is a dataset from \cite{Jiang:2023} containing responses from 11 LLMs. We use a random split of 50\% of the LLMs for training and the rest for testing. We use (exponentiated) BARTScore \citep{YuaNeuLiu2021} as the evaluation metric, following \citet{Jiang:2023}. 
The MixInstruct dataset does not have LLM API costs annotated.  
We use the number of parameters of the LLM as the cost of processing one prompt.

\begin{figure}[h]
  \begin{minipage}[b]{.98\linewidth}
  \centering
  \resizebox{\columnwidth}{!}{
    \begin{tabular}{lrrrrrrrrrrrrrrr@{}}
    \toprule
    \multirow{2}{*}{
    \diaghead(-3,1){justpadspaceeee}{Method}{Dataset}
    } 
    & \multicolumn{3}{c}{EmbedLLM}  
    & \multicolumn{3}{c}{RouterBench} 
    & \multicolumn{3}{c}{SPROUT o3-mini} 
    & \multicolumn{3}{c}{Headlines} 
    & \multicolumn{3}{c}{MixInstruct} 
    \\
    \cmidrule(lr){2-4} \cmidrule(lr){5-7} \cmidrule(lr){8-10}  \cmidrule(lr){11-13} \cmidrule(lr){14-16}
     & Area (50\%) $\uparrow$ & Area  $\uparrow$ & QNC $\downarrow$ 
     & Area (50\%) $\uparrow$ & Area  $\uparrow$ & QNC $\downarrow$ 
     & Area (50\%) $\uparrow$ & Area  $\uparrow$ & QNC $\downarrow$ 
     & Area (50\%) $\uparrow$ & Area  $\uparrow$ & QNC $\downarrow$ 
     & Area (50\%) $\uparrow$ & Area  $\uparrow$ & QNC $\downarrow$ 
     \\
    \midrule
    ZeroRouter
       & .285\sig & .607\sig & 87.5\%\sig
       & .320\sig & .689\sig & 99.9\%\nosig
       & .404\sig & .820\sig & 100.0\%\sig
      & .0233\sig & .0487\sig & 96.8\%\nosig
      & .380\sig & .819\sig & 88.0\%\sig
      \\
      
    $K$-NN
       & .298\sig & .636\sig & 46.1\%\sig
        & .328\sig & .707\sig & 99.7\%\nosig 
        & .418\sig & .844\sig & 29.6\%\sig
        & .0235\sig & \best{.0494}\nosig & \best{94.8\%}\nosig
        & \best{.411}\nosig & .830\sig & 43.7\%\sig
        \\
     \ourMethod{} ($K$-Means) 
        & .307\sig & .648\sig & 33.9\%\nosig
        & \best{.332}\nosig & \best{.712}\nosig & \best{99.4}\%\nosig 
        & \best{.421}\nosig & \best{.850}\nosig & \best{19.6}\%\nosig
        & .0235\sig & .0491\nosig & 95.2\%\nosig
        & .409\sig & .828\sig & 56.9\%\sig
        \\
    \ourMethod{} (LearnedMap)
        & \best{.308}\nosig & \best{.651}\nosig & \best{33.2\%}\nosig
        & .331\nosig & .711\nosig & 99.6\%\nosig 
        & .420\nosig & .846\nosig & 23.4\%\nosig
        & \best{.0236}\nosig & .0491\nosig & 96.2\%\nosig
        & \best{.411}\nosig & \best{.832}\nosig & \best{34.9}\%\nosig
        \\
    \midrule
       MLP (Clairvoyant)
        & .314\nosig & .664\nosig & 26.9\%\nosig
        & .339\nosig & .723\nosig & 95.2\%\nosig 
        & .427\nosig & .859\nosig & 4.5\%\nosig
        & .0240\nosig & .0502\nosig & 85.7\%\nosig
        & .412\nosig & .835\nosig & 18.1\%\nosig
        \\
    Matrix Fac. (Clairvoyant)
        & .313\nosig & .663\nosig & 27.7\%\nosig
        & .338\nosig & .720\nosig & 99.5\%\nosig
        & .426\nosig & .857\nosig & 5.0\%\nosig
        & .0242\nosig & .0505\nosig & 82.9\%\nosig 
        & .413\nosig & .835\nosig & 13.7\%\nosig
        \\
    \bottomrule
    \end{tabular}
   }
  \end{minipage}
  \begin{minipage}[b]{.98\linewidth}
    \centering
    \begin{subfigure}{0.32\linewidth}
        \includegraphics[width=\textwidth]{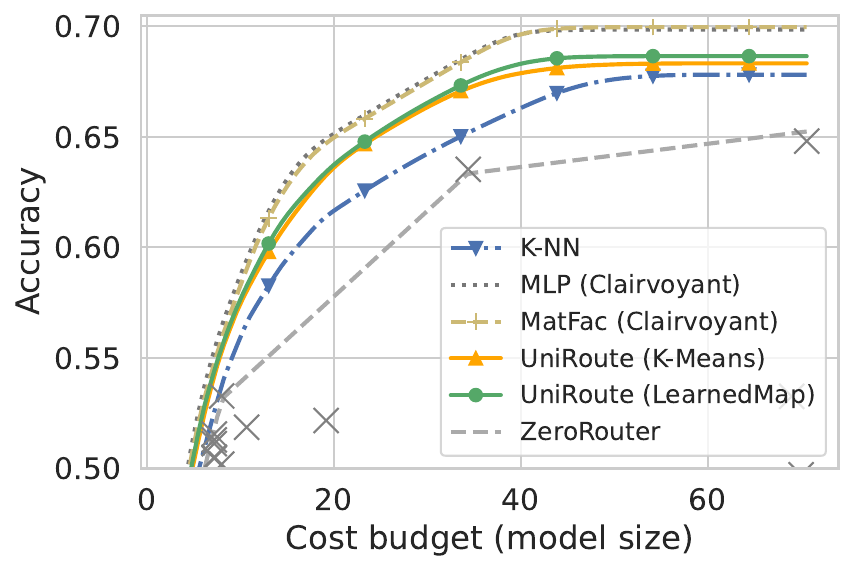}
        \caption{EmbedLLM}
        \label{fig:embedllm}
    \end{subfigure}
    \begin{subfigure}{0.32\linewidth}
        \includegraphics[width=\textwidth]{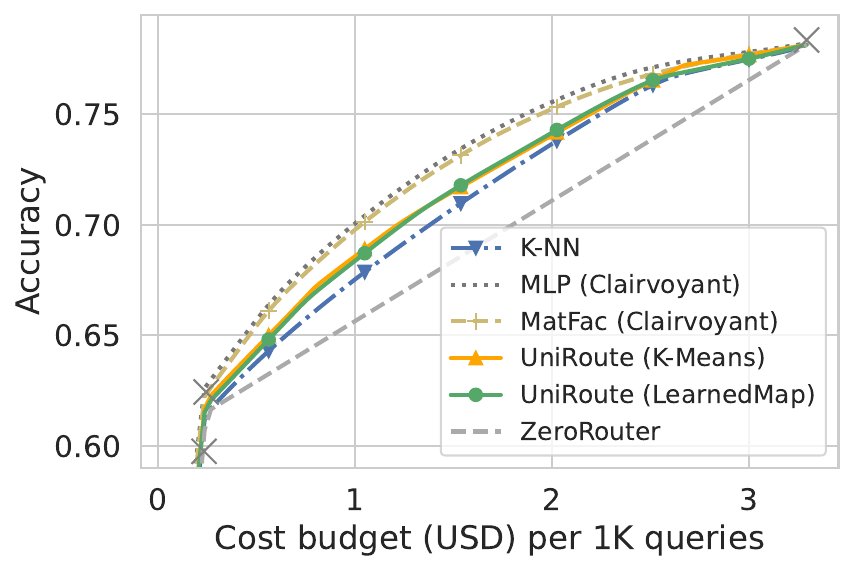}
        \caption{RouterBench}
        \label{fig:routerbench}
    \end{subfigure}
     \begin{subfigure}{0.33\linewidth}
        \includegraphics[width=\textwidth]{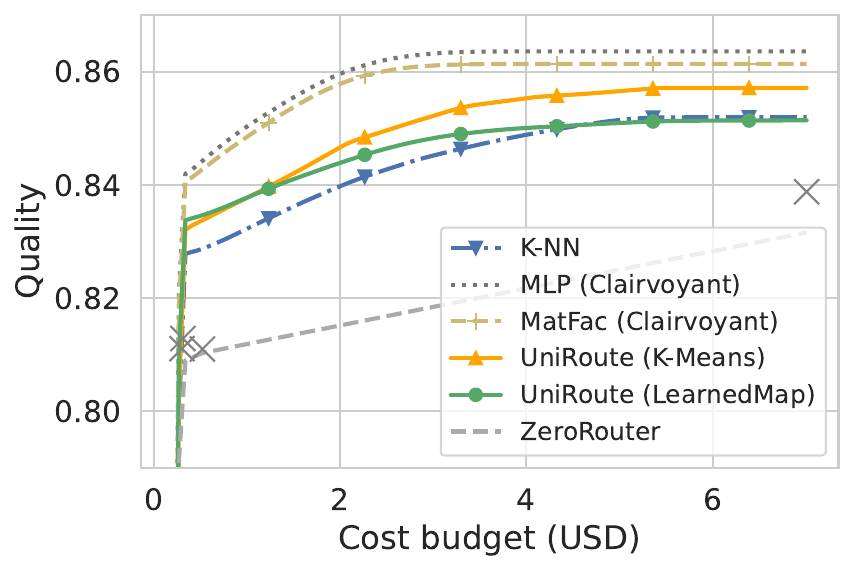}
        \caption{SPROUT o3-mini}
        \label{fig:sprout}
    \end{subfigure}
    
    \begin{subfigure}{0.32\linewidth}
        \includegraphics[width=\textwidth]{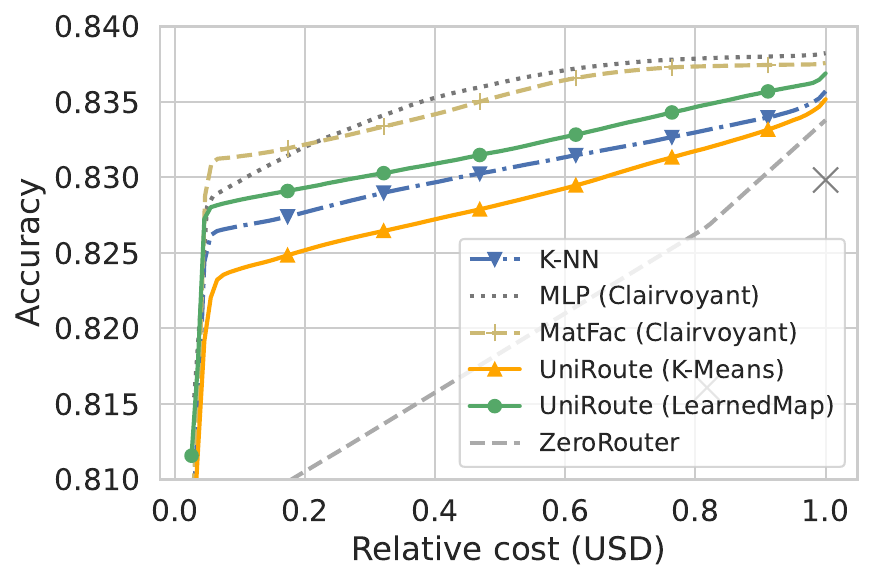}
        \caption{Headlines}
        \label{fig:headlines}
    \end{subfigure}
    \begin{subfigure}{0.33\linewidth}
        \includegraphics[width=\textwidth]{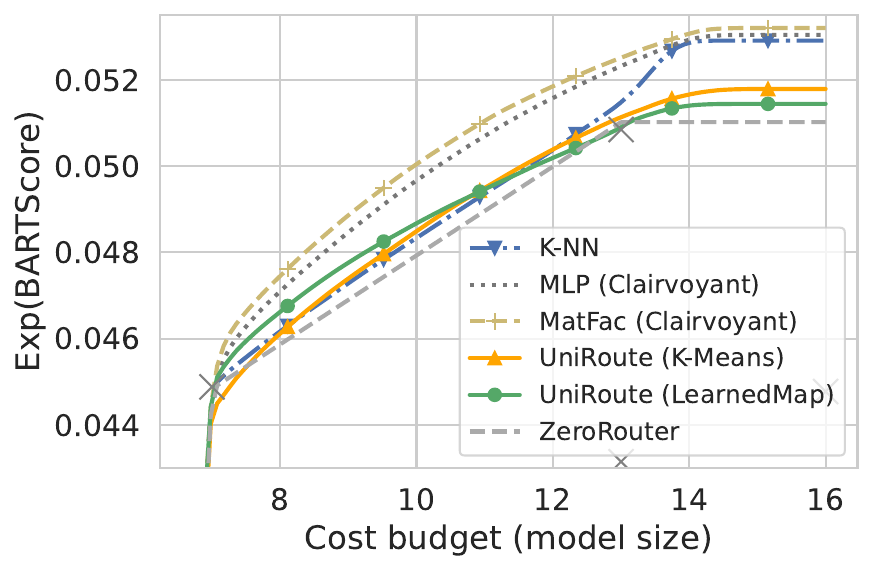}
        \caption{MixInstruct}
        \label{fig:mixinstruct}
    \end{subfigure}
  \end{minipage}
    \caption{Deferral curves and router evaluation metrics (Area (50\%), Area, and QNC) for different methods in the dynamic pool setting. 
    MLP \citep{Hu:2024b} and MatFac \citep{OngAlmWu2024,ZhuWuWen2024}, are oracle methods that observe testing LLMs during training. 
    ZeroRouter~\citep{HuBieLi2024} and $K$-NN~\citet{HuBieLi2024,Shnitzer:2023} are baselines applicable to the dynamic LLM pool setting.
    }
  \label{fig:deferral-curves-app}
\end{figure}

\newpage

\subsection{Train on Chatbot Arena and Test on EmbedLLM}
\label{sec:chatbot_exp_details}
We observe that representing each prompt with a small number of binary attributes that capture its inherent hardness shines when there is a prompt distribution shift at test time. 
To illustrate this, we use the seven binary difficulty attributes proposed in \citet{Li:2024c}, and prompt Gemini 1.5 Pro to annotate each attribute for each training prompt. We then construct a query embedder 
\begin{equation}
\boldsymbol{\varphi}(\bx) = \sigma (\boldsymbol{V}^\top \mathrm{Gecko}(\bx))\in [0, 1]^7, 
\label{eq:attribute_embed}
\end{equation}
where $\boldsymbol{V} \in \mathbb{R}^{768 \times 7}$ is distilled using the training set to predict the 7-category attributes for any new prompt $\bx$.

We compare  Gecko-based prompt representation and attribute-based representation by training on Chatbot Arena conversation data \citep{ZheChiShe2023}, and testing on EmbedLLM, which contains mostly Q\&A prompts.
To reduce confounding factors, we train on all LLMs that are present in both datasets (26 LLMs), and test on the same set of LLMs (i.e., no unseen LLMs at test time).   After appropriate filtering, the Chatbot Arena dataset has 8447 records left. The filtering step ensures that we only deal with LLMs that are present in both datasets. These examples are split further into 90\% training and 10\% validation splits.

The Chatbot Arena dataset contains pairwise comparison labels: each user prompt is responded to by two random LLMs, to which the user selects the better response. To evaluate per-cluster performance for representing each LLM, we fit the Bradley-Terry-Luce model \citep{bradley1952rank} to the pairwise comparison labels within a cluster and estimate the pointwise quality scores for each LLM for that cluster. We use the full EmbedLLM dataset for testing.

The results are shown in \cref{fig:chatbot_arena} where we compare two variants of our proposed \ourMethod{} ($K$-means):  
\begin{enumerate}
    \item \emph{K-means (Gecko)}. This is \ourMethod{} ($K$-means) (see \Cref{sec:cluster_router}) where the text embedding is set to the Gecko model \citep{LeeDaiRen2024}. 
    \item \emph{K-means (Attributes)}. This is \ourMethod{} ($K$-means) where the text embedding is set to  \eqref{eq:attribute_embed}.
\end{enumerate}
We observe that K-means (Attributes) in this case performs better than K-means (Gecko), suggesting that using prompt hardness attributes helps improve robustness to prompt distribution shifts. 
In fact, this routing approach is the only method that can reach the performance of the most accurate model in the pool, thus attaining a finite quality-neutral cost (QNC). The reason the Pareto-random router has a decreasing trend is because the Pareto-optimal LLMs are chosen using the validation set, and turn out to be not optimal for the test set.

\begin{figure}[!t]
\vspace{4mm}
  \begin{minipage}[t]{.49\linewidth}
  \vspace{-\topskip}
\includegraphics[width=0.9\textwidth]{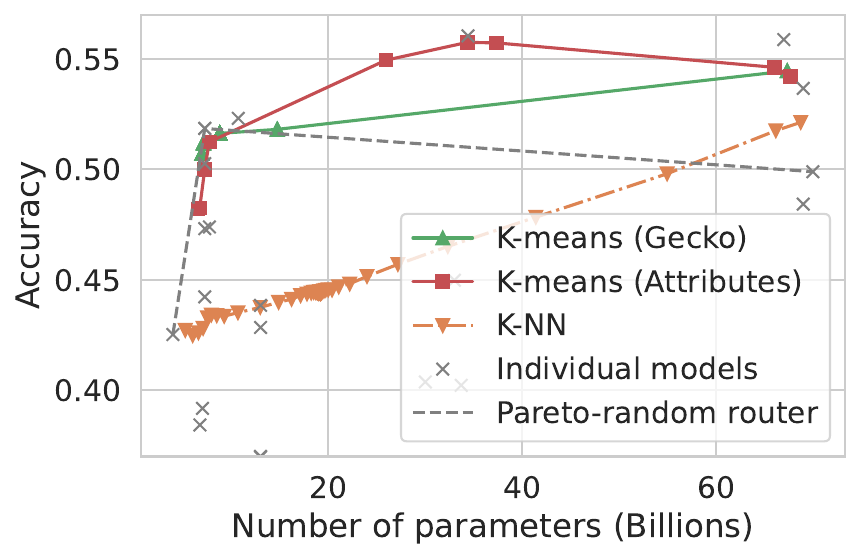}
  \end{minipage}%
  \hfill%
  \begin{minipage}[t]{.49\linewidth}
  \vspace{-\topskip}
\begin{tabular}{lrrr}
\toprule
Method & Area$\uparrow$ & QNC$\downarrow$ & Peak Acc.$\uparrow$\\
\midrule
Pareto-random router & .507 & $\infty$ & 51.9\% \\
K-NN & .472 & $\infty$ & 52.1\% \\
K-means (Gecko) & .529 & $\infty$ & 54.5\% \\
K-means (Attributes) & \best{.545} & \best{.97} & \best{55.8\%} \\
\bottomrule
\end{tabular}
  \end{minipage}
\caption{Deferral curves of routers trained on Chatbot Arena with pairwise comparison labels, and tested on EmbedLLM with per-prompt correctness labels.}
\label{fig:chatbot_arena}
\end{figure}

\subsection{Static LLM Pool Setting}
\label{app:static-pool}
While the 
dynamic pool setting is the focal point of our work,
we show in Table \ref{tab:static-pool} 
that even in the \emph{static} LLM pool setting,
\ourMethod{} is typically comparable to most baselines. The MLP baseline alone often has a slight edge. This is because, unlike our methods, which are tied to a particular input LLM representation, the MLP approach has the flexibility of learning its own representation for each fixed LLM. Note that this is possible only in the static LLM pool setting.

In this case, all LLMs are seen during training, and hence the training and validation samples have correctness label annotations from all LLMs. We tune the hyper-parameters, such as $K$ in $K$-NN and $K$-Means, the number of hidden nodes in MLP, and the intermediate dimension in Matrix Factorization, to maximize the area under the deferral curve on the validation sample.  We pick $K$ from the range $\{5, 10, 25, 50, 75, 100, 150, 200, 250, 300\}$ for $K$-NN and $K$-Means, pruning out values that are too large for the given validation sample size.

For \ourMethod{} (LearnedMap), we used two hidden layers, with the same chosen number of hidden nodes as MLP. We set the number of clusters to be the same as the chosen number of clusters for \ourMethod{} (K-Means). We employed Adam with learning rate 0.001 and batch size 64 for training, picking the checkpoint with the best quality metric on the validation set. For MixInstruct alone, given that the dataset is prone to overfitting, we replicate the same hyper-parameter choices as in the dynamic LLM pool setting (Appendix \ref{app:implementation-details}).

For these experiments, we additionally consider the Headlines dataset (see \Cref{app:deferral}) where all LLMs are fully observed during training.  We additionally consider the LiveCodeBench coding benchmark \citep{jain2024livecodebench}, and include from it 15 LLMs for which the model size was publicly available (the Math+Code dataset contains two LLMs from this benchmark). We split the LiveCodeBench dataset into 60\% for training, 10\% for validation, and 30\% for testing, and employ the same hyper-parameter choices as in MixInstruct. For Headlines, we use 40\%, 10\%, 50\% for training, validation, and testing, respectively.

\begin{table}[H]
\centering
\caption{%
Comparing \ourMethod{} with existing methods for the \textbf{static LLM pool} setting, where $\msetTrain = \msetNew$. We report the area under the deferral curve ($\uparrow$). The best baseline results are \best{{highlighted}}, and the best \ourMethod{} results are \textbf{boldfaced}. Even in the static setting, our approach is competitive compared to most  baselines. The MLP baseline often has a slight edge. This is because, unlike our methods, which are tied to a particular input LLM representation, the MLP approach has the flexibility of learning its own representation for each fixed LLM. Note that this is possible only in the static LLM pool setting, and not the dynamic setting, which is the focus of this paper.
}
\resizebox{\linewidth}{!}{
\begin{tabular}{lcccccc}
\toprule
\diaghead(-3,1){justpadspaceeee}{Method}{Dataset} & \rotatebox[origin=c]{50}{EmbedLLM} &  \rotatebox[origin=c]{50}{MixInstruct}  & \rotatebox[origin=c]{50}{RouterBench} &
\rotatebox[origin=c]{50}{Math+Code} & \rotatebox[origin=c]{50}{LiveCodeBench} & 
\rotatebox[origin=c]{50}{Headlines} 
\\
\midrule
ZeroRouter~\citep{HuBieLi2024} & .601 & .0483 & .707 & .392
& .457 & .834
\\
MLP~\citet{HuBieLi2024} & \best{{.689}} & .0500 & \best{{.747}} & \best{{.483}}
& .480 & .852 
\\
Matrix Factorization~\citep{OngAlmWu2024,ZhuWuWen2024} & .682 & .0503 & .744
& .482
& \best{{.482}} & .849
\\
$K$-NN~\citet{HuBieLi2024,Shnitzer:2023} & .636 & \best{{.0510}} & .744  
& .475
& .474
&  \best{{.854}}
\\
\midrule
\ourMethod{} ($K$-Means) & .682 & \textbf{.0504} & {.744} 
&
\textbf{.480}
&
\textbf{.481}
& .845
\\
\ourMethod{} (LearnedMap) & \textbf{.683} & {.0502} & \textbf{.744} 
& .463
& .479 
& \textbf{.854}
\\
\bottomrule
\end{tabular}
}
\label{tab:static-pool}
\end{table}

\subsection{Visualisation of LLM Embeddings}
\label{sec:embed-viz}

We visualise the cluster-based
LLM embeddings $\LLMEmbed$ learned from our clustering procedure on the EmbedLLM and RouterBench datasets in Figures~\ref{fig:embed-similarity-embedllm} and~\ref{fig:embed-similarity-routerbench}.
These heatmaps illustrate the Gaussian kernel similarity between all pairs of LLM embeddings on these datasets.
The results are largely intuitive:
e.g.,
on RouterBench, we find that the \texttt{claude} family of models are highly similar to each other.
Similarly,
on EmbedLLM,
code-focussed models generally tend to demonstrate a high degree of similarity.

We emphasise again that EmbedLLM previously also considered representing LLMs as feature vectors.
Importantly, however, their representation depends on a \emph{fixed} pool of LLMs,
and does not readily generalise to new LLMs (without further retraining).

\begin{figure}[h]
    \centering
    \resizebox{0.9\linewidth}{!}{%
    \includegraphics[width=0.99\textwidth]{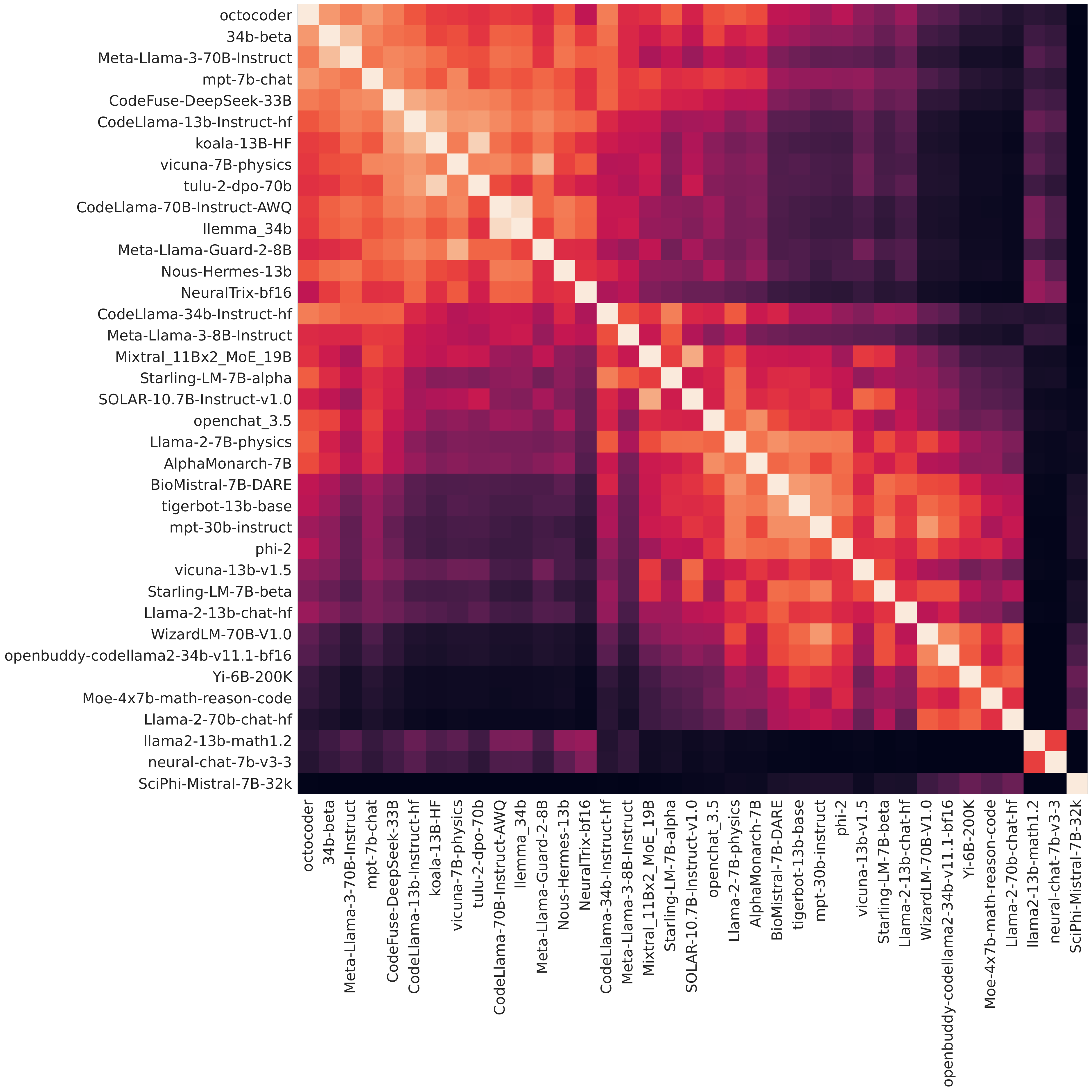}
    }
    
    \caption{Gaussian kernel similarity between pairs of LLM embeddings on EmbedLLM dataset.}
    \label{fig:embed-similarity-embedllm}
\end{figure}

\begin{figure}[H]
    \centering
    \resizebox{0.5\linewidth}{!}{%
    \includegraphics[scale=0.5]{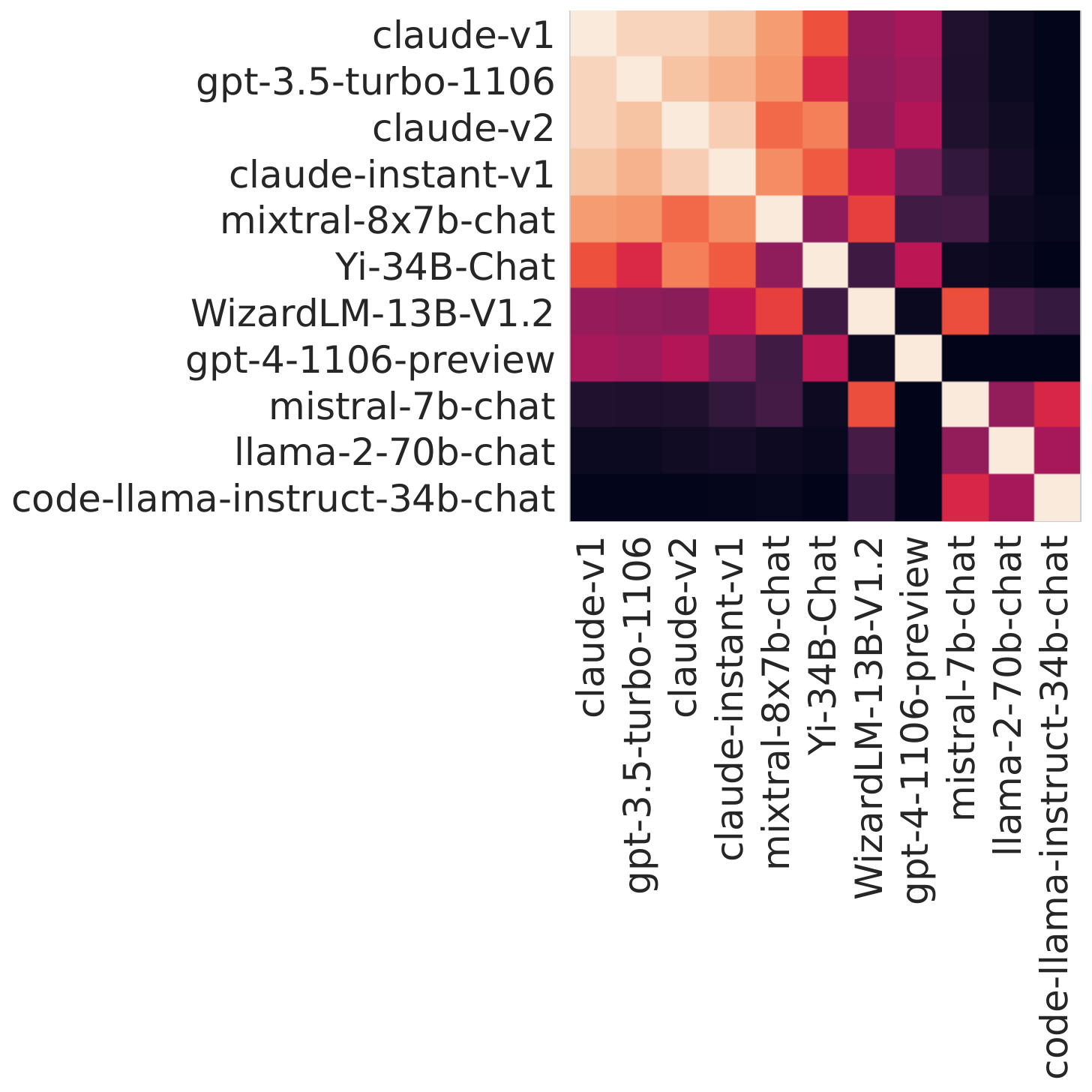}%
    }
    
    \caption{Gaussian kernel similarity between pairs of LLM embeddings on RouterBench dataset.}
    \label{fig:embed-similarity-routerbench}
\end{figure}

\ifarxiv
\else
\section{Additional Related Work}
\label{app:related}

\label{sec:discussion_related}

\textbf{Model representation}.
Broadly, the idea of representing models via compact representations has been studied in various contexts~\citep{Achille:2019,Yan:2020,Cotler:2023}.
In the context of LLMs,
our proposal relates to several recent strands of work~\citep{Tailor:2024,Thrush:2024,ZhuWuWen2024,Feng:2024,Li:2025,Zhao:2024}, which merit individual discussion.

The idea of representing an LLM via a \LossVector{} has close relation to some recent works.
In~\citet{Tailor:2024}, the authors proposed to represent ``experts'' via predictions on a small \emph{context set},
so as to enable deferral to a new, randomly selected expert at evaluation time.
While not developed in the context of LLMs (and for a slightly different problem),
this bears similarity to our proposal of representing LLMs via a \LossVector{} on a validation set; 
however, note that the mechanics of routing based on this vector (via suitable clustering) are novel.
\citet{Thrush:2024} considered representing LLMs via their perplexity on a set of public benchmarks,
for subsequent usage in pre-training data selection.
This shares the core idea of using LLM performance on a set of examples to enable subsequent modelling, albeit for a wholly different task.
\citet{ZhuWuWen2024} proposed to construct a generic
embedding for LLMs based on performance on public benchmarks.
This embedding is constructed via a form of matrix factorisation, akin to~\citet{OngAlmWu2024}.
While~\citet{ZhuWuWen2024} discuss model routing as a possible use-case for such embeddings, there is no explicit evaluation of embedding generation in the case of \emph{dynamic} LLMs;
note that this setting would na\"{i}vely require full re-computation of the embeddings, or some version of incremental matrix factorisation~\citep{Brand:2002}.

Some recent works have considered the problem of routing with a dynamic pool of LLMs.
\citet{Feng:2024} proposed a graph neural network approach,
wherein LLMs are related to prompts and \emph{tasks} (e.g., individual benchmarks that a prompt is drawn from).
Such 
pre-defined \emph{task labels} for input prompts
may be unavailable in some practical settings.
\citet{Li:2025} proposed to use LLM performance on benchmark data to construct a \emph{model identity vector},
which is trained using a form of variational inference.
This has conceptual similarity with our \LossVector{} proposal (and the works noted above);
however, our mechanics of learning based on this vector 
(i.e., the cluster-based representation)
are markedly different.
Further,~\citet{Li:2025} consider an online routing setting wherein bandit approaches are advisable,
whereas we consider and analyse a conventional supervised learning setting.
Finally, we note that~\citet{Zhao:2024} consider the problem of routing to a dynamic \emph{LoRA pool}, where LoRA modules are natively represented by aggregating (learned) embeddings on a small subset of training examples.
These embeddings are learned via a contrastive loss,
constructed based on certain pre-defined task labels.
As such labels may not be provided in many settings,~\citet{Chen:2024} implemented a variant wherein these are replaced by unsupervised cluster assignments.
While similar in spirit to our proposals, the details of the mechanics are different; e.g., we use the clustering to compute a set of average accuracy scores for each LLM,
rather than training an additional embedding.

Finally, we note that~\citet{Chen:2024} also consider the use of clustering as part of router training.
However, the usage is fundamentally different:
~\citet{Chen:2024}
regularise the learned prompt embedding $\QueryEmbed$
based on
cluster information,
while we use clustering to construct the LLM embedding $\LLMEmbed$.

Our proposed approach is similar in spirit to methods such as  K-NN, where the prediction for a prompt is based on the labels associated with queries in its neighborhood; in our case, we consider pre-defined clusters instead of example-specific neighborhoods.

\fi

\end{document}